\documentclass[11pt]{article}
\usepackage{graphicx}
\usepackage{fullpage}
\usepackage{amsmath,amssymb,amsthm}
\usepackage{enumerate}
\usepackage{mathrsfs}
\usepackage{marginnote} %for math mode margin notes.
\usepackage[round]{natbib}

\usepackage[T1]{fontenc}
\usepackage{kpfonts}
\usepackage{microtype}

\usepackage{enumitem}

\numberwithin{equation}{section}

\usepackage[usenames]{color}
\definecolor{plum}  {rgb}{.4,0,.4}
\definecolor{BrickRed} {rgb}{0.6,0,0}

\usepackage{commath}
\usepackage[normalem]{ulem}

\usepackage[plainpages=false,pdfpagelabels,colorlinks=true,linkcolor=BrickRed,citecolor=plum]{hyperref}

\def\ddefloop#1{\ifx\ddefloop#1\else\ddef{#1}\expandafter\ddefloop\fi}

\def\ddef#1{\expandafter\def\csname c#1\endcsname{\ensuremath{\mathcal{#1}}}}
\ddefloop ABCDEFGHIJKLMNOPQRSTUVWXYZ\ddefloop

\def\ddef#1{\expandafter\def\csname s#1\endcsname{\ensuremath{\mathsf{#1}}}}
\ddefloop ABCDEFGHIJKLMNOPQRSTUVWXYZ\ddefloop

\def\E{\mathbf{E}}

\def\Reals{\mathbb{R}}
\def\Naturals{\mathbb{N}}

\def\Wass{\cW}

\def\deq{:=}

\def\wh#1{\widehat{#1}}
\def\bd#1{\mathbf{#1}}
\def\bz{\bd{z}}
\def\bZ{\bd{Z}}
\def\bP{\bd{P}}
\def\Osc{{\mathrm{Osc}}}
\def\tO{{\tilde{\cO}}}
\def\tOm{\tilde{\Omega}}

\def\d{{\mathrm d}}

\def\1{{\mathbf 1}}

\def\ave#1{\langle #1 \rangle}
\def\eps{\varepsilon}

\newtheorem{theorem}{Theorem}[section]
\newtheorem{lemma}{Lemma}[section]
\newtheorem{proposition}{Proposition}[section]

\newtheorem{remark}{Remark}[section]

\usepackage{cleveref}

\begin{document}

\title{Non-Convex Learning via Stochastic Gradient Langevin Dynamics: A Nonasymptotic Analysis}

\author{Maxim Raginsky\thanks{University of Illinois; {\tt maxim@illinois.edu}. Research  supported in part by the NSF under CAREER
award CCF-1254041, and in part by the Center for Science of Information (CSoI), an
NSF Science and Technology Center, under grant agreement CCF-0939370.} \and Alexander Rakhlin\thanks{University of Pennsylvania; {\tt rakhlin@wharton.upenn.edu}. Research supported in part by the NSF under grant no.~CDS\&E-MSS 1521529.} \and Matus Telgarsky\thanks{University of Illinois, Simons Institute; {\tt mjt@illinois.edu}.}}

\date{}

\maketitle
\thispagestyle{empty}

\begin{abstract} 	
  Stochastic Gradient Langevin Dynamics (SGLD) is a popular variant of Stochastic Gradient Descent, where properly scaled isotropic Gaussian noise is added to an unbiased estimate of the gradient at each iteration. This modest change allows SGLD to escape local minima and suffices to guarantee asymptotic convergence to global minimizers for sufficiently regular non-convex objectives \citep{Gelfand_Mitter_globalopt}.

The present work provides a nonasymptotic analysis in the context of non-convex learning problems, giving finite-time guarantees for SGLD to find approximate minimizers of both empirical and population risks. 

  As in the asymptotic setting, our analysis relates the discrete-time SGLD Markov chain to a
  continuous-time diffusion process. A new tool that drives the results is the use of weighted transportation cost inequalities to quantify the rate of convergence of SGLD to a stationary distribution in the Euclidean $2$-Wasserstein distance.
\end{abstract}

\section{Introduction and informal summary of results}

Consider a stochastic optimization problem
\begin{align*}
	\text{minimize} \qquad  
	F(w) \deq \E_P[f(w,Z)] = \int_\sZ f(w,z) P(\d z),
\end{align*}
where $w$ takes values in $\Reals^d$ and $Z$ is a random element of some space $\sZ$ with an unknown probability law $P$. We have access to an $n$-tuple $\bZ = (Z_1,\ldots,Z_n)$ of i.i.d.\ samples drawn from $P$, and our goal is to generate a (possibly random) hypothesis $\wh{W} \in \Reals^d$ with small expected excess risk
\begin{align}
	\label{eq:excess_risk}
	\E F(\wh{W}) - F^*,
\end{align}
where $F^* \deq \inf_{w \in \Reals^d} F(w)$, and the expectation is with respect to the training data $\bZ$ and any additional randomness used by the algorithm for generating $\wh{W}$. 

When the functions $w\mapsto f(w,z)$ are not convex, theoretical analysis of  global convergence becomes largely intractable. On the other hand, non-convex optimization is currently witnessing an impressive string of empirical successes, most notably in the realm of deep neural networks. Towards the aim of bridging this gap between theory and practice, this paper provides a theoretical justification for \emph{Stochastic Gradient Langevin Dynamics} (SGLD), a popular variant of stochastic gradient descent, in which properly scaled isotropic Gaussian noise is added to an unbiased estimate of the gradient at each iteration  \citep{Gelfand_Mitter_globalopt,Borkar_Mitter_Langevin,welling2011bayesian}. 

Since the population distribution $P$ is unknown, we attempt to (approximately) minimize 
\begin{align}\label{eq:empirical_risk}
	F_\bz(w) \deq \frac{1}{n}\sum^n_{i=1}f(w,z_i),
\end{align}
the empirical risk of a hypothesis $w \in \Reals^d$ on a dataset $\bz = (z_1,\ldots,z_n) \in \sZ^n$. The SGLD algorithm studied in this work is given by the recursion
\begin{align}\label{eq:Langevin0}
	W_{k+1} = W_k - \eta g_k + \sqrt{2\eta \beta^{-1}} \xi_k
\end{align}
where $g_k$ is a conditionally unbiased estimate of the gradient $\nabla F_{\bz}(W_k)$, $\xi_k$ is a standard Gaussian random vector in $\Reals^d$, $\eta > 0$ is the step size, and $\beta > 0$ is the inverse temperature parameter. Our analysis begins with the standard observation (see, e.g., \citet{Borkar_Mitter_Langevin} for a rigorous treatment or \citet{welling2011bayesian} for a heuristic discussion) that the discrete-time Markov process \eqref{eq:Langevin0} can be viewed as a discretization of the continuous-time Langevin diffusion described by the It\^o stochastic differential equation
	\begin{align}\label{eq:diffusion}
		\d W(t) = - \nabla F_\bz(W(t)) \d t + \sqrt{2\beta^{-1}} \d B(t), \qquad t \ge 0
	\end{align}
where $\{B(t)\}_{t \ge 0}$ is the standard Brownian motion in $\Reals^d$. Under suitable assumptions on $f$, it can be shown that the Gibbs measure $\pi_\bz(\d w) \propto\exp(-\beta F_\bz(w))$ is the unique invariant distribution of \eqref{eq:diffusion}, and that the distributions of $W(t)$ converge rapidly to $\pi_\bz$ as $t \to \infty$ \citep{Chiang_etal}. Moreover, for all sufficiently large values of $\beta$, the Gibbs distribution concentrates around the minimizers of $F_\bz$ \citep{Hwang_Laplace}. Consequently, a draw from the Gibbs distribution is, with high probability, an almost-minimizer of the empirical risk \eqref{eq:empirical_risk}, and, if one can show that the SGLD recursion tracks the Langevin diffusion in a suitable sense, then it follows that the distributions of $W_k$ will be close to the Gibbs measure for all sufficiently large $k$. Hence, one can argue that, for large enough $k$, the output of SGLD is also an almost-minimizer of the empirical risk.

It is well-recognized, however, that minimization of the empirical risk $F_\bz$ does not immediately translate into minimization of the population risk $F$. A standard approach for addressing the issue is to decompose the excess risk into a  sum of two terms,  $F(\wh{W})-F_\bz(\wh{W})$ (the generalization error of $\wh{W}$) and  $F_\bz(\wh{W})-F^*$ (the gap between the empirical risk of $\wh{W}$ and the minimum of the population risk), and then show that both of these terms are small (either in expectation or with high probability). Taking $\wh{W}=W_k$ and letting $\wh{W}^*$ be the output of the Gibbs algorithm under which the conditional distribution of $\wh{W}^*$ given $\bZ = \bz$ is equal to $\pi_\bz$, we decompose the excess risk \eqref{eq:excess_risk} as follows:
\begin{align}
	\label{eq:excess_risk_decomposition}
	\E F(\wh{W}) - F^* = \big( \E F(\wh{W}) - \E F(\wh{W}^*) \big) + \big( \E F(\wh{W}^*) - \E F_{\bZ} (\wh{W}^*)\big) + \big(\E F_{\bZ} (\wh{W}^*)-  F^*\big),
\end{align}
where the first term is the difference of expected population risks of SGLD and the Gibbs algorithm, the second term is the generalization error of the Gibbs algorithm, and the third term is easily upper-bounded in terms of expected suboptimality $\E \big(F_{\bZ} (\wh{W}^*) - \min_{w} F_{\bZ}(w)\big)$ of the Gibbs algorithm for the empirical risk. Observe that only the first term pertains to SGLD, whereas the other two involve solely the Gibbs distribution. The main contribution of this work is in showing finite-time convergence of SGLD for a non-convex objective function. Informally, we can state our main result as follows:
\begin{enumerate}
	\item For any $\eps > 0$, the first term in \eqref{eq:excess_risk_decomposition} scales as
	\begin{align}\label{eq:T1}
		\eps \cdot {\rm Poly}\left(\beta,d,\frac{1}{\lambda_*}\right) \quad \text{for } k \succeq {\rm Poly}\left(\beta, d, \frac{1}{\lambda_*}\right) \cdot \frac{1}{\eps^4} ~\text{ and }~ \eta \le \left(\frac{\eps}{\log (1/\eps)}\right)^4,
	\end{align}
	where $\lambda_*$ is a certain spectral gap parameter that governs the exponential rate of convergence of the Langevin diffusion to its stationary distribution. This spectral gap parameter itself might depend on $\beta$ and $d$, but is independent of $n$.
	\item The second and third terms in \eqref{eq:excess_risk_decomposition} scale, respectively, as
	\begin{align}\label{eq:T2}
\frac{(\beta+d)^2}{\lambda_* n} \text{, }\quad \frac{d\log (\beta+1)}{\beta}.
\end{align}
\end{enumerate} 
\subsection{Method of analysis: an overview}
Our analysis draws  heavily on the theory of optimal transportation \citep{Villani_topics} and on the analysis of Markov diffusion operators \citep{Bakry_Gentil_Ledoux_book} (the necessary background on Markov semigroups and functional inequalities is given in Appendix~\ref{app:background}). In particular, we control the convergence of SGLD to the Gibbs distribution in terms of \textit{$2$-Wasserstein distance}
\begin{align*}
	\Wass_2(\mu,\nu) \deq \inf \left\{ (\E \|V-W\|^2)^{1/2} : \mu = \cL(V), \nu = \cL(W)\right\},
\end{align*}
where $\| \cdot \|$ is the Euclidean $(\ell^2)$ norm on $\Reals^d$, $\mu$ and $\nu$ are Borel probability measures on $\Reals^d$ with finite second moments, and the infimum is taken over all random couples $(V,W)$ taking values in $\Reals^d \times \Reals^d$ with marginals $V \sim \mu$ and $W \sim \nu$. 

To control the first term on the right-hand side of \eqref{eq:excess_risk_decomposition}, we first upper-bound the $2$-Wasserstein distance between the distributions of $W_k$ (the $k$th iterate of SGLD) and $W(k\eta)$ (the point reached by the Langevin diffusion at time $t=k\eta$). This requires some heavy lifting: Existing bounds on the $2$-Wasserstein distance between a diffusion process and its time-discretized version due to \citet{Alfonsi_etal_Euler} scale like $\eta e^{k\eta}$, which is far too crude for our purposes. By contrast, we take an indirect route via a Girsanov-type change of measure and a weighted transportation-cost inequality of \citet{Bolley_Villani_weighted_Pinsker} to obtain a bound that scales like $k\eta \cdot \eta^{1/4}$. This step relies crucially on a certain exponential integrability property of the Langevin diffusion. Next, we show that the Gibbs distribution satisfies a logarithmic Sobolev inequality, which allows us to conclude that the $2$-Wasserstein distance between the distribution of $W(k\eta)$ and the Gibbs distribution decays exponentially as $e^{-k\eta}$. Since $\Wass_2$ satisfies the triangle inequality, we can produce an upper bound on the first term in \eqref{eq:excess_risk_decomposition} that scales as $k\eta \cdot \eta^{1/4} + e^{-k\eta}$. This immediately suggests that we can make this term as small as we wish by first choosing  a large enough horizon $t = k\eta$ and then a small enough step size $\eta$. Overall, this leads to the bounds stated in \eqref{eq:T1}.

To control the second term in \eqref{eq:excess_risk_decomposition}, we show that the Gibbs algorithm is \textit{stable} in $2$-Wasserstein distance with respect to local perturbations of the training dataset. This step, again, relies on the logarithmic Sobolev inequality for the Gibbs distribution. To control the third term, we use a nonasymptotic Laplace integral approximation to show that a single draw from the Gibbs distribution is an approximate minimizer of the empirical risk.  We use a Wasserstein continuity result due to \citet{Polyanskiy_Wu_Wasserstein} and a well-known equivalence between stability of empirical minimization and generalization \citep{Mukherjee2006,rakhlin2005stability} to show that, in fact, the Gibbs algorithm samples from near-minimizers of the population risk.

We remark that our result readily extends to the case when the stochastic gradients $g_k$ in \eqref{eq:Langevin0} are formed with respect to independent draws from the data-generating distribution $P$ -- e.g., when taking a single pass through the dataset. In this case, the target of optimization is $F$ itself rather than $F_{\bz}$, and we simply omit the second term in \eqref{eq:excess_risk_decomposition}. If the main concern is not consistency (as in \eqref{eq:excess_risk}) but rather the generalization performance of SGLD itself, then the same analysis applied to the decomposition
\begin{align}
	\label{eq:generalization_decomposition}
	\E F_{\bZ}(\wh{W}) - \E F(\wh{W}) = \big( \E F_{\bZ}(\wh{W}) - \E F_{\bZ}(\wh{W}^*) \big) + \big( \E F_{\bZ}(\wh{W}^*) - \E F (\wh{W}^*)\big) + \big(\E F(\wh{W}^*) - \E F(\wh{W})\big)
\end{align}
gives an upper bound of \eqref{eq:T1} plus the first term of \eqref{eq:T2}. In other words, while the rate of \eqref{eq:excess_risk} may be hampered by the slow convergence of $\frac{d\log \beta}{\beta}$, the rate of generalization is not. Finally, if each data point is used only once, the generalization performance is controlled by \eqref{eq:T1} alone.

\subsection{Related work}

The asymptotic study of convergence of discretized Langevin diffusions for non-convex objectives has a long history, starting with the work of  \citet{Gelfand_Mitter_globalopt}. Most of the work has focused on annealing-type schemes, where both the step size $\eta$ and the temperature $1/\beta$ are decreased with time.  \cite{marquez1997convergence} and \citet{pelletier1998weak} studied the rates of weak convergence for both the Langevin diffusion and the discrete-time updates. However, when $\eta$ and $\beta$ are kept fixed, the updates do not converge to a global minimizer, but one can still aim for convergence to a stationary distribution. An asymptotic study of this convergence, in the sense of relative entropy, was initiated by \citet{Borkar_Mitter_Langevin}. 

\sloppypar \citet{Dalalyan_Tsybakov_LMC} and \citet{Dalalyan_LMC} analyzed rates of convergence of discrete-time Langevin updates (with exact gradients) in the case of convex functions, and provided nonasymptotic rates of convergence in the total variation distance for sampling from log-concave densities. \citet{durmus2015non} refined these results by establishing geometric convergence in total variation distance for convex and strongly convex objective functions, and provided some results for non-convex objectives that can be represented as a bounded perturbation of a convex or a strongly convex function. \citet{bubeck2015sampling} studied projected Langevin updates in the convex case. 

Our work is motivated in part by recent papers on non-convex optimization and, in particular, on optimization problems related to neural networks. A heuristic analysis of SGLD was given by \citet{welling2011bayesian}, and a modification of SGLD to improve generalization performance was recently proposed by \citet{Chaudhari16}. Deliberate addition of noise was also proposed by \citet{ge2015escaping} as a strategy for escaping from saddle points, and \citet{belloni2015escaping} analyzed a simulated anealing method based on Hit-and-Run for sampling from nearly log-concave distributions. While these methods aim at avoiding local minima through randomn perturbations, the line of work on continuation methods and graduated optimization  \citep{hazan2015graduated} attempts to create sequences of smoothed approximations that can successively localize the optimum.

\cite{hardt2015train} studied uniform stability and generalization properties of stochastic gradient descent with both convex and non-convex objectives. For the non-convex case, their upper bound on stability degrades with the number of steps of the optimization procedure, which was taken by the authors as a prescription for early stopping. In contrast, we show that, under our assumptions, non-convexity does not imply loss of stability when the latter is measured in terms of $2$-Wasserstein distance to the stationary distribution. In addition, we use the fact that Gibbs distribution concentrates on approximate empirical  minimizers, implying convergence for the \emph{population} risk via stability \citep{rakhlin2005stability,Mukherjee2006}.

\section{The main result}

We begin by giving a precise description of the SGLD recursion. A \textit{stochastic gradient oracle}, i.e., the mechanism for accessing the gradient of $F_\bz$ at each iteration, consists of a collection $(Q_\bz)_{\bz \in \sZ^n}$ of probability measures on some space $\sU$ and a mapping $g : \Reals^d \times \sU \to \Reals^d$, such that, for every $\bz \in \sZ^n$,
\begin{align}\label{eq:unbiased_grad}
	\E g(w,U_\bz) = \nabla F_\bz(w), \qquad \forall w \in \Reals^d
\end{align}
where $U_\bz$ is a random element of $\sU$ with probability law $Q_\bz$. Conditionally on $\bZ=\bz$, the SGLD update takes the form
\begin{align}\label{eq:Langevin}
	W_{k+1} = W_k - \eta g(W_k,U_{\bz,k}) + \sqrt{2\eta \beta^{-1}} \xi_k, \qquad k = 0,1,2,\ldots
\end{align}
where $\{U_{\bz,k}\}^\infty_{k=0}$ is a sequence of i.i.d.\ random elements of $\sU$ with probability law $Q_\bz$ and $\{\xi_k\}^\infty_{k=0}$ is a sequence of i.i.d.\ standard Gaussian random vectors in $\Reals^d$. We assume that $W_0$, $(\bZ,\{U_{\bZ,k}\}^\infty_{k=0})$, and $\{\xi_k\}^\infty_{k=0}$ are mutually independent. We impose the following assumptions (see the discussion in Section~\ref{ssec:assumptions} for additional details):
\begin{enumerate}[label={\bf (A.\arabic*)}]
	\item The function $f$ takes nonnegative real values, and there exist constants $A, B \ge 0$, such that
	\begin{align*}
		|f(0,z)| \le A \qquad \text{and} \qquad \| \nabla f(0,z) \| \le B \qquad \forall z \in \sZ.
		\end{align*}
    \item For each $z \in \sZ$, the function $f(\cdot,z)$ is \textit{$M$-smooth}: for some $M > 0$,
	\begin{align*}
	\|\nabla f(w,z)-\nabla f(v,z)\| \le M \| w - v\|, \qquad \forall w,v \in \Reals^d.
\end{align*}
    \item For each $z \in \sZ$, the function $f(\cdot,z)$ is \textit{$(m,b)$-dissipative} \citep{Hale_book}: for some $m > 0$ and $b \ge 0$,
	\begin{align}\label{eq:dissipation}
		\ave{w, \nabla f(w,z)} \ge m \| w \|^2 - b, \qquad \forall w \in \Reals^d.
	\end{align}
	\item There exists a constant $\delta \in [0,1)$, such that, for each $\bz \in \sZ^n$,\footnote{We are reusing the constants $M$ and $B$ from {\bf (A.1)} and {\bf (A.2)} in \eqref{eq:SG_quadratic_growth} mainly out of considerations of technical convenience; any other constants $M',B'>0$ can be substituted in their place without affecting the results.}
	\begin{align}\label{eq:SG_quadratic_growth}
		\E[\|g(w,U_\bz)-\nabla F_\bz(w)\|^2] \le 2\delta\left(M^2 \| w\|^2 + B^2\right), \qquad \forall w \in \Reals^d.
	\end{align}
	\item The probability law $\mu_0$ of the initial hypothesis $W_0$ has a bounded and strictly positive density $p_0$ with respect to the Lebesgue measure on $\Reals^d$, and
	\begin{align*}
		\kappa_0 \deq \log \int_{\Reals^d} e^{\| w\|^2} p_0(w) \d w < \infty.
	\end{align*}
\end{enumerate}
We are now ready to state our main result. A crucial role will be played by the \textit{uniform spectral gap}
\begin{align}\label{eq:uniform_spectral_gap}
	\lambda_* \deq \inf_{\bz \in \sZ^n}\inf \left\{ \frac{\int_{\Reals^d}\|\nabla g\|^2 \d\pi_\bz}{\int_{\Reals^d}g^2 \d\pi_\bz} : g \in C^1(\Reals^d) \cap L^2(\pi_\bz),\, g \neq 0,\, \int_{\Reals^d} g \d\pi_\bz = 0\right\},
\end{align}
where $\pi_\bz(\d w) \propto e^{-\beta F_\bz(w)}\d w$ is the Gibbs distribution. As detailed in Section~\ref{ssec:assumptions}, Assumptions {\bf (A.1)}--{\bf (A.3)} suffice to ensure that $\lambda_* > 0$. In the statement of the theorem, the notation $\tO(\cdot)$ and $\tOm(\cdot)$ gives explicit dependence on the parameters $\beta,\lambda_*$, and $d$, but hides factors that depend (at worst) polynomially on the parameters $A,B,1/m,b,M,\kappa_0$. Explicit expressions for all constants are given in the proof.

\begin{theorem}\label{thm:main} Suppose that the regularity conditions {\bf (A.1)}--{\bf (A.5)} hold. Then, for any $\beta \ge 1 \vee 2/m$ and any $\eps \in (0, \frac{m}{4M^2} \wedge e^{-\tOm(\lambda_*/\beta (d+\beta))})$, the expected excess risk of $W_k$ is bounded by
	\begin{align}\label{eq:SGLD_excess_risk}
		\E F(W_k) - F^* \le \tO\left( \frac{\beta(\beta+d)^2}{\lambda_*} \left(\delta^{1/4}\log\left(\frac{1}{\eps}\right)+\eps\right) + \frac{(\beta+d)^2 }{\lambda_* n} + \frac{d \log (\beta+1)}{\beta}\right),
	\end{align}
provided
	\begin{align}\label{eq:SGLD_iteration_complexity}
		k = \tOm\left(\frac{\beta(d+\beta)}{\lambda_*\eps^4}\log^5\left(\frac{1}{\eps}\right)\right) \quad \text{and} \quad \eta \le \left( \frac{\eps}{\log (1/\eps)}\right)^4.
	\end{align}
\end{theorem}

\section{Proof of Theorem~\ref{thm:main}}
\label{sec:proof}

\subsection{A high-level overview}

Let $\mu_{\bz,k} \deq \cL(W_k|\bZ=\bz)$, $\nu_{\bz,t} \deq \cL(W(t)|\bZ=\bz)$, and $\E_\bz[\cdot] \deq \E[\cdot|\bZ=\bz]$. In a nutshell, our proof of Theorem~\ref{thm:main} consists of the following steps:
\begin{enumerate}
	\item We first show that, for all sufficiently small $\eta > 0$, the SGLD recursion  \eqref{eq:Langevin} tracks the continuous-time Langevin diffusion process \eqref{eq:diffusion} in $2$-Wasserstein distance:
	\begin{align}\label{eq:diffusion_approx_0}
		\Wass_2(\mu_{\bz,k},\nu_{\bz,k\eta}) = \tO\left((\beta+d)(\delta^{1/4}+\eta^{1/4})k\eta\right)
	\end{align}
(the precise statement with explicit constants is given in Proposition~\ref{prop:diffusion_approx}).
\item Next, we show that the Langevin diffusion \eqref{eq:diffusion} converges exponentially fast to the Gibbs distribution $\pi_\bz$:
\begin{align*}
	\Wass_2(\nu_{\bz,k\eta},\pi_\bz) = \tO\left(\frac{\beta+d}{\sqrt{\lambda_*}}\right)e^{-\tOm(\lambda_*k\eta/\beta(d+\beta))}.
\end{align*}
This, together with \eqref{eq:diffusion_approx_0} and the triangle inequality, yields the estimate
\begin{align}\label{eq:Wass_to_equil_0}
	\Wass_2(\mu_{\bz,k},\pi_\bz) = \tO\left((\beta+d)(\delta^{1/4}+\eta^{1/4})k\eta\right) + \tO\left(\frac{\beta+d}{\sqrt{\lambda_*}}\right)e^{-\tOm(\lambda_*k\eta/\beta(d+\beta))}
\end{align}
(see Proposition~\ref{prop:Wass_to_equil} for explicit constants). Observe that there are two terms on the right-hand side of \eqref{eq:Wass_to_equil_0}, one of which grows linearly with $t=k\eta$, while the other one decays exponentially with $t$. Thus, we can first choose $t$ large enough and then $\eta$ small enough, so that
\begin{align}\label{eq:Wass_to_equil_1}
	\Wass_2(\mu_{\bz,k},\pi_\bz) = \tO\left(\frac{\beta(d+\beta)^2}{\lambda_*}\left(\delta^{1/4}\log\left(\frac{1}{\eps}\right)+\eps\right)\right).
\end{align}
The resulting choices of $t=k\eta$ and $\eta$ translate into the expressions for $k$ and $\eta$ given in \eqref{eq:C1_C2}.
\item The upshot of Eq.~\eqref{eq:Wass_to_equil_1} is that, for large enough $k$, the conditional probability law of $W_k$ given $\bZ=\bz$ is close, in $2$-Wasserstein, to the Gibbs distribution $\pi_\bz$. Thus, we are led to consider the \textit{Gibbs algorithm} that generates a random draw from $\pi_\bz$. We show that the resulting hypothesis is an almost-minimizer of the empirical risk, i.e.,
\begin{align}\label{eq:Gibbs_AERM_0}
	\int_{\Reals^d} F_\bz(w)\pi_\bz(\d w) - \min_{w \in \Reals^d} F_\bz(w) = \tO\left(\frac{d}{\beta}\log\frac{\beta+1}{d}\right)
\end{align}
(see Proposition~\ref{prop:almost_ERM} for the exact statement), and also that the Gibbs algorithm is stable in the $2$-Wasserstein distance: for any two datasets $\bz,\bar\bz$ that differ in a single coordinate,
\begin{align*}%\label{eq:Gibbs_stability_0}
	\Wass_2(\pi_\bz,\pi_{\bar\bz}) = \tO\left(\frac{\beta(d+\beta)\sqrt{1+d/\beta}}{\lambda_* n}\right).
\end{align*}
This estimate, together with Lemma~\ref{lm:Wass_continuity} below, implies that the Gibbs algorithm is uniformly stable \citep{bousquet2002stability}:
\begin{align}\label{eq:Gibbs_stability_0}
	\sup_{z \in \sZ} \left|\int_{\Reals^d} f(w,z)\pi_\bz(\d w)-\int_{\Reals^d} f(w,z)\pi_{\bar\bz}(\d w)\right| = \tO\left(\frac{(\beta+d)^2}{\lambda_*n}\right)
\end{align}
(see Proposition~\ref{prop:Gibbs_stability}). The almost-ERM property \eqref{eq:Gibbs_AERM_0} and the uniform stability propery \eqref{eq:Gibbs_stability_0}, together with \eqref{eq:Wass_to_equil_1}, yield the statement of the theorem.
\end{enumerate}

\subsection{Technical lemmas}
\label{ssec:lemmas}

We first collect a few lemmas that will be used in the sequel; see Appendix~\ref{app:lemmas} for the proofs.

\begin{lemma}[quadratic bounds on $f$]\label{lm:quadratic_bounds} For all $w \in \Reals^d$ and $z \in \sZ$, 
	\begin{align}\label{eq:quadratic_growth}
		\| \nabla f(w,z) \| \le M \|w\| + B
	\end{align}
and
	\begin{align}\label{eq:quadratic_sandwich}
		\frac{m}{3}\|w\|^2 - \frac{b}{2}\log 3 \le f(w,z) \le \frac{M}{2}\|w\|^2 + B\|w\|+A.
	\end{align}
\end{lemma}
\begin{lemma}[uniform $L^2$ bounds on SGLD and Langevin diffusion]\label{lm:mean_square_bounds} For all $0 < \eta < 1 \wedge \frac{m}{4M^2}$ and all $\bz \in \sZ^n$,
	\begin{align}\label{eq:Langevin_bounded}
		\sup_{k \ge 0} \E_\bz\|W_k\|^2 \le \kappa_0 + 2\left(1 \vee \frac{1}{m}\right) \left(b+2B^2+\frac{d}{\beta}\right).
	\end{align}
	and
\begin{align}
	 \E_\bz\|W(t)\|^2 &\le \kappa_0 e^{-2mt}  + \frac{b+d/\beta}{m}\left(1-e^{-2mt}\right) \label{eq:2nd_moment_t}\\
	&\le \kappa_0 + \frac{b+d/\beta}{m}. \label{eq:diffusion_bounded}
\end{align}
\end{lemma}
\begin{lemma}[exponential integrability of Langevin diffusion]\label{lm:diffusion_MGF} For all $\beta \ge 2/m$, we have
\begin{align}\label{eq:diffusion_MGF}
	\log\E_\bz\big[e^{\|W(t)\|^2}\big] \le \kappa_0 +2 \left(b+\frac{d}{\beta}\right)t.
\end{align}
\end{lemma}
\begin{lemma}[relative entropy bound]\label{lm:relent_bound} For any $w \in \Reals^d$ and any $\bz \in \sZ^n$, 
	\begin{align}\label{eq:relent_bound}
		D(\mu_0 \| \pi_{\bz}) \le  \log \|p_0\|_\infty  + \frac{d}{2}\log \frac{3\pi}{m\beta}  + \beta\left(\frac{M\kappa_0}{3} + B\sqrt{\kappa_0} + A + \frac{b}{2}\log 3 \right). 
	\end{align}
\end{lemma}

\begin{lemma}[2-Wasserstein continuity for functions of quadratic growth, \citet{Polyanskiy_Wu_Wasserstein}]\label{lm:Wass_continuity} Let $\mu,\nu$ be two probability measures on $\Reals^d$ with finite second moments, and let $g : \Reals^d \to \Reals$ be a $C^1$ function obeying
	\begin{align}\label{eq:PW_regular}
		\| \nabla g(w) \| \le c_1 \| w \| + c_2, \qquad \forall w \in \Reals^d
	\end{align}
for some constants $c_1 > 0$ and $c_2 \ge 0$.
Then
\begin{align}
	\left|\int_{\Reals^d} g\, \d\mu-\int_{\Reals^d} g\,\d\nu \right| \le (c_1 \sigma + c_2) \Wass_2(\mu,\nu)
\end{align}
where $\sigma^2 \deq \int_{\Reals^d} \mu(\d w) \| w \|^2 \vee \int_{\Reals^d} \nu(\d w) \| w\|^2$.
\end{lemma} 

\subsection{The diffusion approximation}
\label{ssec:diffusion_approx}

\sloppypar Recall that $\mu_{\bz,k} = \cL(W_k|\bZ=\bz)$ and $\nu_{\bz,t} = \cL(W(t)|\bZ=\bz)$, and we take $\mu_{\bz,0} = \nu_{\bz,0}=\mu_0$. In this section, we derive an upper bound on the 2-Wasserstein distance $\Wass_2(\mu_{\bz,k},\nu_{\bz,k\eta})$. The analysis consists of two steps. We first upper-bound the relative entropy $D(\mu_{\bz,k}\|\nu_{\bz,k\eta})$ via a change-of-measure argument following  \citet{Dalalyan_Tsybakov_LMC} (see also \citet{Dalalyan_LMC}), except that we also have to deal with the error introduced by the stochastic gradient oracle. We then use a weighted transportation-cost inequality of \citet{Bolley_Villani_weighted_Pinsker} to control the Wasserstein distance $\Wass_2(\mu_{\bz,k},\nu_{\bz,k\eta})$ in terms of $D(\mu_{\bz,k}\|\nu_{\bz,k\eta})$.

The proof of the following lemma is somewhat lengthy, and is given in Appendix~\ref{app:diffusion_approx}:
\begin{lemma}\label{lm:Girsanov_bound} For any $k \in \Naturals$ and any $\eta \in (0,1 \wedge \frac{m}{4M^2})$, we have
	\begin{align*}
		D(\mu_{\bz,k}\|\nu_{\bz,k\eta}) \le \left(C_0 \beta \delta + C_1 \eta\right)k\eta, 
	\end{align*}
with
	\begin{align*}
		C_0 = \left(M^2 \left(\kappa_0 + 2\left(1 \vee \frac{1}{m}\right)  \left(b+2B^2+\frac{d}{\beta}\right) \right ) + B^2\right), \qquad
		C_1 = 6M^2\left(\beta C_0 + d\right).
	\end{align*}
\end{lemma}
\noindent We now use the following result of \citet[Cor.~2.3]{Bolley_Villani_weighted_Pinsker}: For any two Borel probability measures $\mu,\nu$ on $\Reals^d$ with finite second moments,
\begin{align*}
	\Wass_2(\mu,\nu) \le C_\nu\Bigg[\sqrt{D(\mu \| \nu)} + \left(\frac{D(\mu\|\nu)}{2}\right)^{1/4}\Bigg],
\end{align*}
where
\begin{align*}
	C_\nu = 2 \inf_{\lambda > 0}\left( \frac{1}{\lambda}\left(\frac{3}{2}+\log \int_{\Reals^d} e^{\lambda \|w \|^2}\nu(\d w)\right)\right)^{1/2}.
\end{align*}
Let $\mu = \mu_{\bz,k}$, $\nu = \nu_{\bz,k\eta}$, and take $\lambda =1$. Suppose $k\eta \ge 1$. Since $\beta \ge \frac{2}{m}$, we can use Lemma~\ref{lm:diffusion_MGF} to write
\begin{align*}
	\Wass^2_2(\mu_{\bz,k},\nu_{\bz,k\eta}) &\le \left(12 + 8\left(\kappa_0+2b+\frac{2d}{\beta}\right)k\eta\right) \cdot \left( D(\mu_{\bz,k}\|\nu_{\bz,k\eta})+\sqrt{D(\mu_{\bz,k}\|\nu_{\bz,k\eta})}\right) .
\end{align*}
Moreover, for all $k$ and $\eta$ satisfying the conditions of Lemma~\ref{lm:Girsanov_bound}, plus the additional requirement $k\eta \ge 1$, we can write
\begin{align*}
	D(\mu_{\bz,k}\|\nu_{\bz,k\eta})+\sqrt{D(\mu_{\bz,k}\|\nu_{\bz,k\eta})} &\le \left(C_1+\sqrt{C_1}\right) k\eta^{3/2} + \left(\beta C_0+\sqrt{\beta C_0}\right) \cdot k\eta \sqrt{\delta}.
\end{align*}
Putting everything together, we obtain the following result:
\begin{proposition}\label{prop:diffusion_approx}
For any $k \in \Naturals$ and any $\eta \in (0,1 \wedge \frac{m}{4M^2})$ obeying $k\eta \ge 1$, we have
\begin{align}
  \Wass^2_2(\mu_{\bz,k},\nu_{\bz,k\eta}) \le \Big(\tilde{C}^2_0  \sqrt{\delta} + \tilde{C}^2_1  \sqrt{\eta} \Big) \cdot (k\eta)^2,\label{eq:discretization_error}
\end{align}
with
\begin{align*}
	\tilde{C}^2_0 \deq \left(12 + 8\left(\kappa_0+2b+\frac{2d}{\beta}\right)\right) \left(\beta C_0+\sqrt{\beta C_0}\right) \qquad \text{and} \qquad
	\tilde{C}^2_1 \deq \left(12 + 8\left(\kappa_0+2b+\frac{2d}{\beta}\right)\right)\left(C_1+\sqrt{C_1}\right).
\end{align*}
\end{proposition}

\subsection{Wasserstein distance to the Gibbs distribution}

We now fix a time $t \ge 0$ and examine the 2-Wasserstein distance $\Wass_2(\nu_{\bz,t},\pi_\bz)$. At this point, we need to use a number of concepts from the analysis of Markov diffusion operators; see Appendix~\ref{app:background} for the requisite background. We start by showing the following:
\begin{proposition}\label{prop:LSI} For $\beta \ge 2/m$, all of the the Gibbs measures $\pi_\bz$ satisfy a logarithmic Sobolev inequality with constant
	\begin{align*}
		c_{\rm LS} &\le \frac{2m^2 + 8M^2}{m^2 M \beta} + \frac{1}{\lambda_*}\left( \frac{6M(d+\beta)}{m} + 2\right).
	\end{align*}
\end{proposition}
\noindent Therefore,
\begin{align}\label{eq:TC}
	\Wass_2(\mu,\pi_\bz) \le \sqrt{2 c_{\rm LS} D(\mu \| \pi_\bz)}
\end{align}
by the Otto--Villani theorem, and, since $D(\nu_{\bz,0}\|\pi_\bz) = D(\mu_0 \| \pi_\bz) < \infty$ by Lemma~\ref{lm:relent_bound}, we also have
\begin{align}\label{eq:Wasserstein_decay}
	D(\nu_{\bz,t} \| \pi_\bz) \le D(\mu_0 \| \pi_\bz) e^{-2t/\beta c_{\rm LS}}.
\end{align}
by the theorem on exponential decay of entropy. Combining Eqs.~\eqref{eq:TC} (with $\mu = \nu_{\bz,t}$) and \eqref{eq:Wasserstein_decay} and using Lemma~\ref{lm:relent_bound}, we get
\begin{align*}%\label{eq:terms_1_3}
	\Wass_2(\nu_{\bz,t},\pi_\bz) 
	&\le \sqrt{2 c_{\rm LS} \left( \log \|p_0\|_\infty  + \frac{d}{2}\log \frac{3\pi}{m\beta}  + \beta\left(\frac{M\kappa_0}{3} + B\sqrt{\kappa_0} + A + \frac{b}{2}\log 3 \right)\right)} e^{-t/\beta c_{\rm LS}} \\
	& =: \tilde{C}_2 e^{-t/\beta c_{\rm LS}}.
\end{align*}
Letting $t = k \eta$ and invoking Proposition~\ref{prop:diffusion_approx}, we obtain the following:

\begin{proposition}\label{prop:Wass_to_equil} For all $k$ and $\eta$ satisfying the conditions of Proposition~\ref{prop:diffusion_approx}, we have
	\begin{align*}
	\Wass_2(\mu_{\bz,k},\pi_\bz) \le \left(\tilde{C}_0 \delta^{1/4} + \tilde{C}_1 \eta^{1/4}\right) k\eta + \tilde{C}_2 e^{-k\eta/\beta c_{\rm LS}}.
	\end{align*}
\end{proposition}

\subsection{Almost-ERM property of the Gibbs algorithm}

In this section and the next one, we focus on the properties of the Gibbs algorithm that generates a random hypothesis $\wh{W}^*$ with $\cL(\wh{W}^*|\bZ=\bz) = \pi_\bz$. Let $p_\bz(w) = e^{-\beta F_\bz(w)}/\Lambda_\bz$ denote the density of the Gibbs measure $\pi_\bz$ with respect to the Lebesgue measure on $\Reals^d$, where $\Lambda_\bz \deq \int_{\Reals^d}e^{-\beta F_\bz(w)}\d w$ is the normalization constant known as the partition function. We start by writing 
\begin{align}\label{eq:equilibrium_expected_loss}
	 \int_{\Reals^d} F_\bz(w) \pi_\bz(\d w) = \frac{1}{\beta}\left(h(p_\bz) - \log \Lambda_\bz\right),
\end{align}
where
$$
h(p_\bz) = -\int_{\Reals^d}p_\bz(w)\log p_\bz(w) \d w = -\int_{\Reals^d} \frac{e^{-\beta F_\bz(w)}}{\Lambda_\bz} \log \frac{e^{-\beta F_\bz(w)}}{\Lambda_\bz} \d w
$$
is the differential entropy of $p_\bz$ \citep{CovTho06}. To upper-bound $h(p_\bz)$, we  estimate the second moment of $\pi_\bz$. From \eqref{eq:Wasserstein_decay}, it follows that $\Wass_2(\nu_{\bz,t},\pi_\bz) \xrightarrow{t \to \infty} 0$. Since convergence of probability measures in 2-Wasserstein distance is equivalent to weak convergence plus convergence of second moments \citep[Theorem~7.12]{Villani_topics}, we have by \Cref{lm:mean_square_bounds}
\begin{align}\label{eq:equilibrium_2nd_moment}
  \int_{\Reals^d} \|w\|^2 \pi_\bz(\d w) = \lim_{t \to \infty} \int_{\Reals^d} \|w\|^2 \nu_{\bz,t}(\d w) \le \frac{b+d/\beta}{m}.
\end{align}
The differential entropy of a probability density with a finite second moment is upper-bounded by that of a Gaussian density with the same second moment, so we immediately get
\begin{align}\label{eq:equilibrium_diff_entropy}
	h(p_\bz)  \le \frac{d}{2}\log \left( \frac{2\pi e(b+d/\beta)}{md}\right).
\end{align}
Moreover, let $w^*_\bz$ be any point that minimizes $F_\bz(w)$, i.e., $F^*_\bz := \min_{w \in \Reals^d} F_\bz(w) = F_\bz(w^*_\bz)$. Then $\nabla F_\bz(w^*_\bz) = 0$, and, since $F_\bz$ is $M$-smooth, we have $F_\bz(w) - F^*_\bz \le \frac{M}{2}\|w-w^*\|^2$ by Lemma~1.2.3 in \citet{Nesterov_book}. As a consequence, we can lower-bound $\log \Lambda_\bz$ using a Laplace integral approximation:
\begin{align}
	\log \Lambda_\bz &= \log \int_{\Reals^d} e^{-\beta F_\bz(w)} \d w \nonumber\\
	&= -\beta F^*_\bz + \log \int_{\Reals^d} e^{\beta(F^*_\bz - F_\bz(w))} \d w \nonumber \\
	&\ge -\beta F^*_\bz + \log \int_{\Reals^d} e^{-\beta M \| w - w^*_z \|^2/2} \d w \nonumber \\
	&= -\beta F^*_\bz + \frac{d}{2}\log \left(\frac{2\pi}{M\beta}\right). \label{eq:part_fun_LB}
\end{align}
Using Eqs.~\eqref{eq:equilibrium_diff_entropy} and \eqref{eq:part_fun_LB} in \eqref{eq:equilibrium_expected_loss} and simplifying, we obtain the following result:
\begin{proposition}\label{prop:almost_ERM} For any $\beta \ge 2/m$,
\begin{align*}
	\int_{\Reals^d} F_\bz(w) \pi_\bz(\d w) - \min_{w \in \Reals^d} F_\bz(w) 
	&\le \frac{d}{2\beta}\log \left(\frac{eM}{m}\left(\frac{b\beta}{d} + 1\right)\right).
\end{align*}
\end{proposition}

\subsection{Stability of the Gibbs algorithm}

Our last step before the final analysis is to show that the Gibbs algorithm  is \textit{uniformly stable}. Fix two $n$-tuples $\bz=(z_1,\ldots,z_n),\bar\bz=(\bar z_1,\ldots, \bar z_n) \in \sZ^n$ with ${\rm card} |\{ i: z_i \neq \bar{z}_i\}|=1$. Then the Radon--Nikodym derivative $p_{\bz,\bar\bz} = \frac{\d\pi_\bz}{\d\pi_{\bar\bz}}$ can be expressed as
\begin{align*}
	p_{\bz,\bar\bz}(w) = \frac{\exp\left(-\frac{\beta}{n}\big(f(w,z_{i_0})-f(w,\bar z_{i_0})\big)\right)}{\Lambda_{\bz}/\Lambda_{\bar\bz}},
\end{align*}
where $i_0 \in [n]$ is the index of the coordinate where $\bz$ and $\bar\bz$ differ. In particular,
\begin{align*}
	\nabla \sqrt{p_{\bz,\bar\bz}(w)} = \frac{\beta}{2n} \Big(\nabla_w f(w,\bar z_{i_0})-\nabla_w f(w,z_{i_0})\Big) \sqrt{p_{\bz,\bar\bz}(w)}.
\end{align*}
Therefore, since $\pi_{\bar\bz}$ satisfies a logarithmic Sobolev inequality with constant $c_{\rm LS}$ given in Proposition~\ref{prop:LSI}, we can write
\begin{align*}
	D(\pi_\bz \| \pi_{\bar\bz}) &\le 2c_{\rm LS} \int \left\| \nabla \sqrt{p_{\bz,\bar\bz}} \right\|^2 \d\pi_{\bar\bz} \\
	&= \frac{c_{\rm LS}\beta^2}{2n^2} \int_{\Reals^d} \Big\|\nabla_w f(w,\bar z_{i_0})-\nabla_w f(w,z_{i_0}) \Big\|^2 p_{\bz,\bar\bz}(w) \pi_{\bar\bz}(\d w) \\
	&= \frac{c_{\rm LS}\beta^2}{2n^2} \int_{\Reals^d} \Big\|\nabla_w f(w,\bar z_{i_0})-\nabla_w f(w,z_{i_0}) \Big\|^2  \pi_\bz(\d w) \\
	&\le \frac{2c_{\rm LS}\beta^2}{n^2} \left(M^2 \int_{\Reals^d}\|w\|^2\pi_\bz(\d w) + B^2\right),
\end{align*}
where the last line follows from the quadratic growth estimate \eqref{eq:quadratic_growth}. Taking $\mu = \pi_\bz$ in \eqref{eq:TC} and using the above bound and the second-moment estimate \eqref{eq:equilibrium_2nd_moment}, we obtain
\begin{align*}%\label{eq:term_2}
	\Wass_2(\pi_\bz,\pi_{\bar\bz}) \le \frac{2c_{\rm LS}\beta}{n}\sqrt{B^2 + \frac{M^2(b+d/\beta)}{m}}.
\end{align*}
Finally, observe that, for each $z \in \sZ$, the function $w \mapsto f(w,z)$ satisfies the conditions of Lemma~\ref{lm:Wass_continuity} with $c_1 = M$ and $c_2 = B$, while $\pi_\bz$ and $\pi_{\bar\bz}$ satisfy the conditions of Lemma~\ref{lm:Wass_continuity} with $\sigma^2 = \frac{b+d/\beta}{m}$. Thus, we obtain the following uniform stability estimate for the Gibbs algorithm:

\begin{proposition}\label{prop:Gibbs_stability} For any two $\bz,\bar\bz \in \sZ^n$ that differ only in a single coordinate,
	\begin{align*}
		\sup_{z \in \sZ} \left|\int_{\Reals^d} f(w,z)\pi_\bz(\d w) - \int_{\Reals^d} f(w,z) \pi_{\bar\bz}(\d w)\right| \le \frac{\tilde{C}_3}{n}
	\end{align*}
  with 
	\begin{align*}
		\tilde{C}_3 \deq 4 \left(M^2\frac{b+d/\beta}{m}+B^2\right) \beta c_{\rm LS}.
	\end{align*}
\end{proposition}

\subsection{Completing the proof}

Now that we have all the ingredients in place, we can complete the proof of Theorem~\ref{thm:main}. Choose $k \in \Naturals$ and $\eta \in (0,1\wedge\frac{m}{4M^2})$ to satisfy
\begin{align*}
 k\eta = \beta c_{\rm LS} \log \left(\frac{1}{\eps}\right) \qquad \text{and} \qquad \eta \le \left(\frac{\eps}{\log (1/\eps)}\right)^4.
\end{align*}
Then, by Proposition~\ref{prop:Wass_to_equil},
\begin{align*}%\label{eq:Wass_to_equilibrium}
	\Wass_2(\mu_{\bz,k},\pi_\bz) \le \tilde{C}_0\beta c_{\rm LS}\delta^{1/4} \log\left(\frac{1}{\eps}\right)  + \left(\tilde{C}_1\beta c_{\rm LS}+\tilde{C}_2\right) \eps.
\end{align*}
Now consider the random hypotheses $\wh{W}$ and $\wh{W}^*$ with $\cL(\wh{W}|\bZ=\bz) = \mu_{\bz,k}$ and $\cL(\wh{W}^*|\bZ=\bz) = \pi_\bz$. Then
\begin{align*}
	\E F(\wh{W}) - F^* 
& = \E F(\wh{W}) - \E F(\wh{W}^*) + \E F(\wh{W}^*) - F^* \\
& = \int_{\sZ^n} \bP^{\otimes n}(\d\bz) \left(\int_{\Reals^d} F(w) \mu_{k,\bz}(\d w) - \int_{\Reals^d} F(w) \pi_\bz(\d w) \right) + \E F(\wh{W}^*) - F^*.
\end{align*}
The function $F$ satisfies the conditions of Lemma~\ref{lm:Wass_continuity} with $c_1 = M$ and $c_2 = B$, while the probability measures $\mu_{\bz,k}, \pi_\bz$ satisfy the conditions of Lemma~\ref{lm:Wass_continuity} with 
\begin{align*}
	\sigma^2 = \kappa_0 + 2\left(1 \vee \frac{1}{m}\right) \left(b+B^2(1+\delta)+\frac{d}{\beta}\right),
\end{align*}
by Lemma~\ref{lm:mean_square_bounds}. Therefore, for all $\bz \in \sZ^n$,
\begin{align}\label{eq:excess_risk_T1}
	&\int_{\Reals^d} F(w) \mu_{k,\bz}(\d w) - \int_{\Reals^d} F(w) \pi_\bz(\d w) \le K_0 \delta^{1/4}\log \left(\frac{1}{\eps}\right) + K_1 \eps
\end{align}
with
\begin{align*}
	K_0 &\deq \left(M \sqrt{\kappa_0 + 2\left(1 \vee \frac{1}{m}\right) \left(b+2B^2+\frac{d}{\beta}\right)}+B\right) \tilde{C}_0 \beta c_{\rm LS}
\end{align*}
and
\begin{align*}
	K_1 &\deq \left(M \sqrt{\kappa_0 + 2\left(1 \vee \frac{1}{m}\right) \left(b+2B^2+\frac{d}{\beta}\right)}+B\right) (\tilde{C}_1\beta c_{\rm LS} + \tilde{C}_2).
\end{align*}
It remains to analyze the expected excess risk $\E F(\wh{W}^*)-F^*$ of the Gibbs algorithm. To that end, we will use stability-based arguments along the lines of \citet{bousquet2002stability} and \citet{rakhlin2005stability}. We begin by decomposing the excess risk as
\begin{align*}
	\E F(\wh{W}^*)-F^* = \underbrace{\E F(\wh{W}^*) - \E F_{\bZ}(\wh{W}^*)}_{T_1} + \underbrace{\E F_{\bZ}(\wh{W}^*)- F^*}_{T_2}.
\end{align*}
The term $T_1$ is the generalization error of the Gibbs algorithm. To upper-bound it, let $\bZ' = (Z'_1,\ldots,Z'_n) \sim \bP^{\otimes n}$ be independent of $\bZ$ and $\wh{W}^*$. Then
\begin{align}\label{eq:T1_0}
	\E F(\wh{W}^*) - \E F_{\bZ}(\wh{W}^*) &= \E[F_{\bZ'}(\wh{W}^*)-F_\bZ(\wh{W}^*)] \nonumber\\
	&= \frac{1}{n}\sum^n_{i=1}\E[f(\wh{W}^*,Z'_i)-f(\wh{W}^*,Z_i)].
\end{align}
Using the fact that $Z_1,\ldots,Z_n,Z'_1,\ldots,Z'_n$ are i.i.d., as well as the fact that $\bZ'$ is indepdenent of $\wh{W}^*$, the $i$th term in the summation in \eqref{eq:T1_0} can be written out explicitly as follows:
\begin{align}
&	\E[f(\wh{W}^*,Z'_i)-f(\wh{W}^*,Z_i)] \nonumber\\
& = \int_{\sZ^n} \bP^{\otimes n}(\d z) \int_\sZ \bP(\d z'_i) \int_{\Reals^d}\pi_\bz(\d w)\left[f(w,z'_i)-f(w,z_i)\right] \nonumber \\
& = \int_{\sZ^n} \bP^{\otimes n}(\d z_1,\ldots,\d z'_i,\ldots,\d z_n) \int_{\sZ} \bP(\d z_i) \int_{\Reals^d}\pi_{(z_1,\ldots,z'_i,\ldots,z_n)}(\d w) f(w,z_i) \nonumber\\
& \qquad \qquad -  \int_{\sZ^n} \bP^{\otimes n}(\d z_1,\ldots,\d z_i,\ldots,\d z_n) \int_{\sZ} \bP(\d z'_i) \int_{\Reals^d}\pi_{(z_1,\ldots,z_i,\ldots,z_n)}(\d w) f(w,z_i) \nonumber\\
&= \int_{\sZ^n} \bP^{\otimes n}(\d \bz) \int_{\sZ} \bP(\d z'_i) \left(\int_{\Reals^d} \pi_{\bar\bz^{(i)}}(\d w) f(w,z_i) - \int_{\Reals^d} \pi_{\bar\bz}(\d w) f(w,z_i)\right),\label{eq:T1_1}
\end{align}
where $\bar\bz^{(i)} \deq (z_1,\ldots,z_{i-1},z'_i,z_{i+1},\ldots,z_n)$. Noting that $\bar\bz^{(i)}$ and $\bz$ differ only in the $i$th coordinate, we can use Proposition~\ref{prop:Gibbs_stability} to upper-bound the integral in \eqref{eq:T1_1}. Since the resulting estimate is uniform in $i$, from \eqref{eq:T1_0} we obtain
\begin{align}\label{eq:T1_Gibbs}
	\E F(\wh{W}^*) - \E F_{\bZ}(\wh{W}^*) &\le \frac{\tilde{C}_3}{n}.
\end{align}
The term $T_2$ can be handled as follows: Let $w^* \in \Reals^d$ be any minimizer of $F(w)$, i.e., $F(w^*)=F^*$. Then
\begin{align}
	\E F_\bZ(\wh{W}^*) - F^* &= \E \left[F_\bZ(\wh{W}^*)-\min_{w \in \Reals^d} F_\bZ(w)\right] + \E\left[\min_{w \in \Reals^d} F_\bZ(w)-F_\bZ(w^*)\right] \nonumber\\
	&\le \E \left[F_\bZ(\wh{W}^*)-\min_{w \in \Reals^d} F_\bZ(w)\right] \nonumber\\
	&\le \frac{d}{2\beta}\log \left(\frac{eM}{m}\left(\frac{b\beta}{d} + 1\right)\right),\label{eq:T2_Gibbs}
\end{align}
where the last step is by Proposition~\ref{prop:almost_ERM}. From \eqref{eq:T1_Gibbs} and \eqref{eq:T2_Gibbs}, we get
\begin{align}\label{eq:excess_risk_T2}
	\E F(\wh{W}^*) - F^* = \E F(\wh{W}^*) - \E F_{\bZ}(\wh{W}^*) + \E F_{\bZ}(\wh{W}^*)- F^* \le \frac{\tilde{C}_3}{n} + \frac{d}{2\beta}\log \left(\frac{eM}{m}\left(\frac{b\beta}{d} + 1\right)\right).
\end{align}
Combining Eqs.~\eqref{eq:excess_risk_T1} and \eqref{eq:excess_risk_T2}, we obtain the claimed excess risk bound \eqref{eq:SGLD_excess_risk}.

\section{Discussion and directions for future research}
\label{ssec:assumptions}

\paragraph{Regularity assumptions.} The first two assumptions are fairly standard in the literature on non-convex optimization. The dissipativity assumption {\bf (A.3)} merits some discussion. The term ``dissipative'' comes from the theory of dynamical systems \citep{Hale_book,Stuart_Humphries_book}, where it has the following interpretation: Consider the gradient flow described by the ordinary differential equation
\begin{align}\label{eq:ODE_flow}
	\frac{\d w}{\d t} = -\nabla f(w,z), \qquad w(0) =  w_0.
\end{align}
If $f$ is $(m,b)$-dissipative, then a simple argument based on the Gronwall lemma shows that, for any $\eps > 0$ and any initial condition $w_0$, the trajectory of \eqref{eq:ODE_flow} satisfies $\|w(t)\| \le \sqrt{b/m+\eps}$ for all $t \ge \frac{1}{2m}\log \frac{\|w_0\|^2}{\eps}$. In other words, for any $\eps > 0$, the Euclidean ball of radius $\sqrt{b/m+\eps}$ centered at the origin is an \textit{absorbing set} for the flow \eqref{eq:ODE_flow}. If we think of $w(t)$ as the position of a particle moving in $\Reals^d$ in the presence of the potential $f(w,z)$, then the above property means that the particle rapidly loses (or dissipates) energy and stays confined in the absorbing set. However, the behavior of the flow inside this absorbing set may be arbitrarily complicated; in particular, even though \eqref{eq:dissipation} implies that all of the critical points of $w \mapsto f(w,z)$ are contained in the ball of radius $\sqrt{b/m}$ centered at the origin, there can be arbitrarily many such points. The dissipativity assumption seems restrictive, but, in fact, it can be enforced using weight decay regularization \citep{Krogh_weight_decay}. Indeed, consider the regularized objective
\begin{align*}%\label{eq:weight_decay}
f(w,z) = f_0(w,z) + \frac{\gamma}{2}\|w\|^2.
\end{align*}
Then it is not hard to show that, if the function $w \mapsto f_0(w,z)$ is $L$-Lipschitz, then $f$ satisfies {\bf (A.2)} with $m = \gamma/2$ and $b = L^2/2\gamma$. Thus, a byproduct of our analysis is a fine-grained characterization of the impact of weight decay on learning.

Assumption {\bf (A.4)} provides control of the relative mean-square error of the stochastic gradient, viz., $\E\|g(w,U_\bz)\|^2 \preceq (1+\delta)\|\nabla F_\bz(w)\|^2$, and is also easy to satisfy in practice. For example, consider the case where, at each iteration of SGLD, we sample (uniformly with replacement) a random minibatch of size $\ell$. Then we can take $U_\bz = (z_{I_1},\ldots,z_{I_\ell})$, where $I_1,\ldots,I_\ell \stackrel{{\rm i.i.d.}}{\sim} {\rm Uniform}(\{1,\ldots,n\})$, and
\begin{align}\label{eq:minibatch}
g(w,U_\bz) = \frac{1}{\ell}\sum^\ell_{j=1}\nabla f(w,z_{I_j}).
\end{align}
This gradient oracle is clearly unbiased, and a simple calculation shows that {\bf (A.4)} holds with $\delta = 1/\ell$. On the other hand, using the full empirical gradient clearly gives $\delta = 0$.

Finally, the exponential integrability assumption {(\bf A.5)} is satisfied, for example, by the Gaussian initialization $W_0 \sim N(0,\sigma^2 I_d)$ with $\sigma^2 < 1/2$.

\paragraph{Effect of gradient noise and minibatch size selection.} Observe that the excess risk bound \eqref{eq:SGLD_excess_risk} contains a term that goes to zero as $\eps \to 0$, as well as a term that grows as $\log \eps^{-1}$, but goes to zero as the gradient noise level $\delta \to 0$. This suggests selecting the minibatch size 
$$
\ell \ge \frac{1}{\eta} \ge \left(\frac{\log (1/\eps)}{\eps}\right)^4.
$$
to offset the $\log \eps^{-1}$ term.

\paragraph{Uniform spectral gap.} As shown in Appendix~\ref{app:spectral_gap}, Assumptions~{\bf (A.1)}--{\bf (A.3)} are enough to guarantee that the spectral gap $\lambda_*$ is strictly positive. In particular, we give a very conservative estimate
	\begin{align}\label{eq:lambda_LB}
		\frac{1}{\lambda_*} = \tO\left(\frac{1}{\beta(d+\beta)}\right) + \tO\left(1+\frac{d}{\beta}\right)e^{\tO(\beta + d)}.
	\end{align}
Using this estimate in Eq.~\eqref{eq:SGLD_excess_risk}, we end up with a bound on the excess risk that has a dependence on $\exp(\tO(\beta + d))$. This in turn suggests choosing $\eps = 1/n$ and $\beta = \tO(\log n)$; as a consequence, 
the excess risk will decay as $1/\log n$, and the number of iterations
$k$ will scale as $n^{\tO(1)} \exp(\tO(d))$. The alternative regime of conditionally independent stochastic gradients (e.g., using a fresh minibatch at each iteration) amounts to direct optimization of $F$ rather than $F_{\bf z}$ and
suggests the choice of $\beta \approx 1/\eps$. The number of iterations
$k$ will then scale like $\exp(d+1/\eps)$.	

Therefore, in order to apply Theorem~\ref{thm:main}, one needs to
fully exploit the structural properties of the problem at hand
and produce an upper bound on $1/\lambda_*$ which is polynomial in $d$ or even dimension-free. (By contrast, exponential dependence of $1/\lambda_*$ on $\beta$ is unavoidable in the presence of multiple local minima and saddle points; this is a consequence of sharp upper and lower bounds on the spectral gap due to \citet{Bovier_etal_metastability}.)  For example, consider replacing the empirical risk \eqref{eq:empirical_risk} with a smoothed objective
\begin{align*}
	\tilde{F}_\bz(w) &= -\frac{1}{\beta}\log\int_{\{\|v\| \le R\}}  e^{-\beta\gamma \|v-w\|^2/2} e^{-\beta F_\bz(v)}\d v \\
	&= \frac{\gamma}{2} \|w\|^2 - \frac{1}{\beta}\log \int_{\{\|v\| \le R\}} e^{\beta\gamma \ave{v,w} - \beta\gamma \|v\|^2/2}e^{-\beta F_\bz(v)} \d v,
\end{align*}
and running SGLD with $\nabla\tilde{F}_\bz$ instead of $\nabla F_\bz$. Here, $\gamma > 0$ and $R > 0$ are tunable parameters. This modification is closely related to the Entropy-SGD method, recently proposed by \citet{Chaudhari16}. Observe that the modified Gibbs measures $\tilde{\pi}_\bz(\d w) \propto e^{-\beta \tilde{F}_\bz(w)}$ are convolutions of a Gaussian measure and a compactly supported probability measure. In this case, it follows from the results of \citet{Bardet_convolutions} that
\begin{align*}
	\frac{1}{\lambda_*} \le \frac{1}{\beta\gamma}e^{4\beta\gamma R^2}.
\end{align*}
Note that here, in contrast with \eqref{eq:lambda_LB}, this bound is completely dimension-free.  A tantalizing line of future work is, therefore, to find other settings where
$1/\lambda_*$ is indeed small.

\section*{Acknowledgments}

The authors would like to thank Arnak Dalalyan and Ramon van Handel for enlightening discussions.

\bibliography{Langevin.bbl}

\newpage

\appendix

\section{Background on Markov semigroups and functional inequalities}
\label{app:background}

Our analysis relies on the theory of Markov diffusion operators and associated functional inequalities. In this Appendix, we only summarize the key ideas and results; the book by \cite{Bakry_Gentil_Ledoux_book} provides an in-depth exposition.

Let $\{W(t)\}_{t \ge 0}$ be a continuous-time homogeneous Markov process with values in $\Reals^d$, and let $P = \{P_t\}_{t \ge 0}$ be the corresponding Markov semigroup, i.e.,
$$
P_s g(W(t)) = \E[g(W(s+t))|W(t)] 
$$
for all $s,t \ge 0$ and all bounded measurable functions $g : \Reals^d \to \Reals$. (The semigroup law $P_s \circ P_t = P_{s+t}$ is just another way to express the Markov property.) A Borel probability measure $\pi$ is called \textit{stationary} or \textit{invariant} if $\int_{\Reals^d} P_t g\, \d\pi = \int_{\Reals^d} g\, \d\pi$ for all $g$ and $t$. Each $P_t$ can be extended to a bounded linear operator on $L^2(\pi)$, such that $P_t g \ge 0$ whenever $g \ge 0$ and $P_t 1 = 1$ for all $t$. The \textit{generator} of the semigroup is a linear operator $\cL$ defined on a dense subspace $\cD(\cL)$ of $L^2(\pi)$ (the \textit{domain} of $\cL$), such that, for any $g \in \cD(\cL)$, 
$$
\partial_t P_tg = \cL P_t g = P_t \cL g.
$$
In particular, $\cL 1 = 0$, and $\pi$ is an invariant probability measure of the semigroup if and only if $\int_{\Reals^d} \cL g\, \d\pi = 0$ for all $g \in \cD(\cL)$. The generator $\cL$ defines the  \textit{Dirichlet form}
\begin{align}\label{eq:Dirichlet}
\cE(g) \deq -\int_{\Reals^d}g\cL g\,\d\pi.
\end{align}
It can be shown that $\cE(g) \ge 0$, i.e., $-\cL$ is a positive operator (since $\cL 1=0$, zero is an eigenvalue).

Let $P$ be a Markov semigroup with the unique invariant distribution $\pi$ and the Dirichlet form $\cE$. We say that $\pi$ satisfies a \textit{Poincar\'e} (or \textit{spectral gap}) \textit{inequality} with constant $c$ if, for all probability measures $\mu \ll \pi$,
\begin{align}\label{eq:Poincare}
\chi^2(\mu \| \pi) \le c\, \cE\left(\sqrt{\frac{\d\mu}{\d\pi}}\right), 
\end{align}
where $\chi^2(\mu\|\pi) \deq \|\frac{\d\mu}{\d\pi}-1\|^2_{L^2(\pi)}$ is the $\chi^2$ divergence between $\mu$ and $\pi$. The name ``spectral gap" comes from the fact that, if \eqref{eq:Poincare} holds with some constant $c$, then $1/c \ge \lambda$, where
\begin{align*}
\lambda &\deq \inf \left\{ \frac{\cE(g)}{\int_{\Reals^d} g^2\d\pi} : g \in C^2,\, g \neq 0,\, \int_{\Reals^d} g = 0 \right\} \\
&= \inf\left\{ \frac{-\ave{g,\cL g}_{L^2(\pi)}}{\|g\|^2_{L^2(\pi)}} : g \in C^2,\, g \neq 0,\, \int_{\Reals^d} g = 0 \right\}.
\end{align*}
Hence, if $\lambda > 0$, then the spectrum of $-\cL$ is contained in the set $\{0\} \cup [\lambda, \infty)$, so $\lambda$ is the gap between the zero eigenvalue and the rest of the spectrum. We say that $\pi$ satisfies a \textit{logarithmic Sobolev inequality} with constant $c$ if, for all $\mu \ll \pi$,
\begin{align}\label{eq:LSI_0}
	D(\mu \| \pi) \le 2c\,\cE\left(\sqrt{\frac{\d\mu}{\d\pi}}\right),
\end{align}
where $D(\mu\|\pi) = \int \d\mu \log \frac{\d\mu}{\d\pi}$ is the relative entropy (Kullback--Leibler divergence). We record a couple of key consequences of the logarithmic Sobolev inequality. Consider a Markov process $\{W(t)\}_{t \ge 0}$ with a unique invariant distribution $\pi$ and a Dirichlet form $\cE$, such that $\pi$ satisfies a logarithmic Sobolev inequality with constant $c$. Then we have the following:
	\begin{enumerate}
		\item Exponential decay of entropy \citep[Th.~5.2.1]{Bakry_Gentil_Ledoux_book}: Let $\mu_t \deq \cL(W(t))$. Then
	\begin{align}\label{eq:entropy_decay}
		D(\mu_t \| \pi) \le D(\mu_0 \| \pi) e^{-2t/c}.
	\end{align}
	\item Otto--Villani theorem \citep[Th.~9.6.1]{Bakry_Gentil_Ledoux_book}: If $\cE(g) = \alpha \int \| \nabla g \|^2 \d\pi$ for some $\alpha > 0$, then, for any $\mu \ll \pi$,
	\begin{align}\label{eq:Otto_Villani}
		\Wass_2(\mu,\pi) \le \sqrt{2c\alpha D(\mu \|\pi)}.
	\end{align}
	\end{enumerate}
	
Our analysis of SGLD revolves around Markov diffusion processes, so we particularize the above abstract framework to this concrete setting. Let $\{W(t)\}_{t \ge 0}$ be a Markov process evolving in $\Reals^d$ according to an It\^o SDE
\begin{align}\label{eq:SDE}
	\d W(t) = -\nabla H(W(t))\d t + \sqrt{2}\,\d B(t), \qquad t \ge 0
\end{align}
where $H$ is a $C^1$ function and $\{B(t)\}$ is the standard $d$-dimensional Brownian motion. (Replacing the factor $\sqrt{2}$ by $\sqrt{2\beta^{-1}}$ is equivalent to the time rescaling $t \mapsto \beta t$.)  The generator of this semigroup is the second-order differential operator
\begin{align}
	\cL g \deq \Delta g - \ave{\nabla H, \nabla g}
\end{align}
for all $C^2$ functions $g$, where $\Delta \deq \nabla \cdot \nabla$ is the Laplace operator. If the map $w \mapsto \nabla H(w)$ is Lipschitz, then the Gibbs measure $\pi(\d w) \propto e^{-H(w)} \d w$ is the unique invariant measure of the underlying Markov semigroup, and a simple argument using integration by parts shows that the Dirichlet form is given by 
\begin{align}\label{eq:diffusion_Dirichlet}
\cE(g) = \int_{\Reals^d}\| \nabla g\|^2\d\pi.
\end{align}
Thus, the Gibbs measure $\pi$ satisfies a Poincar\'e inequality with constant $c$ if, for any $\mu \ll \pi$,
\begin{align}\label{eq:Poincare_diffusion}
	\chi^2(\mu\|\pi) \le c \int_{\Reals^d}\left\| \nabla\sqrt{\frac{\d\mu}{\d\pi}} \right\|^2 \d\pi
\end{align}
and a logarithmic Sobolev inequality with constant $c$ if
\begin{align}\label{eq:LSI_diffusion}
D(\mu\|\pi) \le 2c \int_{\Reals^d}\left\| \nabla\sqrt{\frac{\d\mu}{\d\pi}} \right\|^2 \d\pi.
\end{align}
If $H$ is $C^2$ and strongly convex, i.e., $\nabla^2 H \succeq K I_d$ for some $K > 0$, then $\pi$ satisfies a logarithmic Sobolev inequality with constant $c = 1/K$. In the absence of convexity, it is in general difficult to obtain upper bounds on Poincar\'e or log-Sobolev constants. The following two propositions give sufficient conditions based on so-called Lyapunov function criteria:

\begin{proposition}[\citet{Bakry_etal_Lyapunov}]\label{prop:Lyapunov_PI} Suppose that there exist constants $\kappa_0,\lambda_0 > 0, R \ge 0$ and a $C^2$ function $V : \Reals^d \to [1,\infty)$ such that
	\begin{align}\label{eq:Lyapunov_PI}
		\frac{\cL V(w)}{V(w)} \le -\lambda_0 + \kappa_0 \1\{\|w\| \le R\}.
	\end{align}
	Then $\pi$ satisfies a Poincar\'e inequality with constant
	\begin{align}\label{eq:Lyapunov_PI_constant}
		c_{\rm P} \le \frac{1}{\lambda_0}\left(1 + C\kappa_0 R^2 e^{\Osc_R(H)}\right),
	\end{align}
	where $C> 0$ is a universal constant and $\Osc_R(H) \deq \max_{\|w\| \le R}H(w)-\min_{\|w\| \le R}H(w)$.
\end{proposition}
\begin{remark} {\em The term involving $\Osc_R(H)$ in \eqref{eq:Lyapunov_PI_constant} arises from a (very crude) estimate of the Poincar\'e constant of the truncated Gibbs measure $\pi_R(\d w) \propto e^{-H(w)}\1\{\|w\| \le R\} \d w$, cf.~the discussion preceding the statement of Theorem~1.4 in \citet{Bakry_etal_Lyapunov}.}
\end{remark}

\begin{proposition}[\citet{Cattiaux_etal_Lyapunov}]\label{prop:Lyapunov_LSI} 
Suppose the following conditions hold:
\begin{enumerate}
	\item There exist constants $\kappa,\gamma > 0$ and a $C^2$ function $V : \Reals^d \to [1,\infty)$ such that
	\begin{align}\label{eq:Lyapunov_LSI}
		\frac{\cL V(w)}{V(w)} \le \kappa - \gamma \|w\|^2
	\end{align}
	for all $w \in \Reals^d$.
	\item $\pi$ satisfies a Poincar\'e inequality with constant $c_{\rm P}$.
	\item There exists some constant $K \ge 0$, such that $\nabla^2 H \succeq -KI_d$.
\end{enumerate}
Let $C_1$ and $C_2$ be defined, for some $\eps > 0$, by
\begin{align*}
	C_1 = \frac{2}{\gamma}\left(\frac{1}{\eps}+\frac{K}{2}\right) + \eps \quad\text{and}\quad C_2= \frac{2}{\gamma}\left(\frac{1}{\eps}+\frac{K}{2}\right)\left(\kappa + \gamma \int_{\Reals^d} \|w\|^2\pi(\d w)\right).
\end{align*}
Then $\pi$ satisfies a logarithmic Sobolev inequality with constant $c_{\rm LS} = C_1 + (C_2+2)c_{\rm P}$.
\end{proposition}
\begin{remark} In particular, if $K \neq 0$, we can take $\eps = 2/K$, in which case
	\begin{align}\label{eq:C1_C2}
		C_1 = \frac{2K}{\gamma} + \frac{2}{K} \quad \text{and} \quad C_2 = \frac{2K}{\lambda}\left(\kappa + \gamma \int_{\Reals^d}\|w\|^2\d\pi\right).
	\end{align}
\end{remark}
 
 \section{A lower bound on the uniform spectral gap}
 \label{app:spectral_gap}

 Our goal here is to prove the crude lower bound on $\lambda_*$ given in Section~\ref{ssec:assumptions}. To that end, we will use the Lyapunov function criterion due to \citet{Bakry_etal_Lyapunov}, which is reproduced as Proposition~\ref{prop:Lyapunov_PI} in Appendix~\ref{app:background}. 
 
We will apply this criterion to the Gibbs distribution $\pi_\bz$ for some $\bz \in \sZ^n$. Thus, we have $H = \beta F_\bz$ and
 $$
 \cL g = \Delta g - \beta \ave{\nabla F_\bz, \nabla g}.
 $$
Consider the candidate Lyapunov function $V(w) = e^{m\beta \|w\|^2/4}$. From the fact that $V \ge 1$ and from the dissipativity assumption {\bf (A.3)}, it follows that
 \begin{align}
 	\cL V(w) &= \left(\frac{m\beta d}{2} + \frac{(m\beta)^2}{4}\|w\|^2 - \frac{m\beta^2}{2} \ave{w, \nabla F_\bz(w)}\right)V(w)\nonumber \\
 	&\le \left(\frac{m\beta (d+b\beta)}{2} - \frac{(m\beta)^2}{4}\|w\|^2\right)V(w). \label{eq:LV}
 \end{align}
Thus, $V$ evidently satisfies \eqref{eq:Lyapunov_PI} with $R^2 = \frac{2\kappa}{\gamma}$, $\kappa_0=\kappa$ and $\lambda_0 = 2\kappa$, where
\begin{align}\label{eq:kappa_gamma}
\kappa \deq \frac{m\beta(d+b\beta)}{2} \qquad \text{and} \qquad \gamma \deq \frac{(m\beta)^2}{4}.
\end{align}
Moreover, from Lemma~\ref{lm:quadratic_bounds} and from the fact that $F_\bz \ge 0$, it follows that 
 \begin{align*}
 	\Osc_R(\beta F_\bz) &\le \beta\left(\frac{MR^2}{2} +  BR + A\right) \le \beta\left(\frac{(M+B)R^2}{2} + A + B \right).
 \end{align*}
 Thus, by Proposition~\ref{prop:Lyapunov_PI}, $\pi_\bz$ satisfies a Poincar\'e inequality with constant
 \begin{align*}%\label{eq:c_P}
 	c_{\rm P} \le \frac{1}{m\beta(d+b\beta)} + \frac{2C(d+b\beta)}{m\beta}\exp\left(\frac{2}{m}(M+B)(b\beta + d) + \beta(A+B)\right).
 \end{align*}
Observe that this bound holds for all $\bz \in \sZ^n$. Using this fact and the relation $1/\lambda \le c_{\rm P}$ between the spectral gap and the Poincar\'e constant, we see that
\begin{align*}
	\frac{1}{\lambda_*} \le \frac{1}{m\beta(d+b\beta)} + \frac{2C(d+b\beta)}{m\beta}\exp\left(\frac{2}{m}(M+B)(b\beta + d) + \beta(A+B)\right),
\end{align*}
which proves the claimed bound.

\section{Proofs for Section~\ref{ssec:lemmas}}
\label{app:lemmas}

\begin{proof}[Proof of Lemma~\ref{lm:quadratic_bounds}] The estimate \eqref{eq:quadratic_growth} is an easy consequence of conditions {\bf (A.1)} and {\bf (A.2)}. Next, observe that, for any two $v,w \in \Reals^d$, 
	\begin{align}\label{eq:unit_speed}
		f(w,z) - f(v,z) &= \int^1_0 \ave{w-v, \nabla f(t w+(1-t)v,z)} \d t.
	\end{align}
In particular, taking $v=0$, we obtain
\begin{align*}
	f(w,z) &= f(0,z) + \int^1_0 \ave{w, \nabla f(tw)} \d t \\
&\stackrel{\rm (i)}{\le} A + \int^1_0  \|w\|\, \|\nabla f(tw,z) \|\, \d t \\%  \label{eq:upper_quad_1}\\
&\stackrel{\rm (ii)}{\le} A+\|w\|\int^1_0 \left(Mt \|w\| + B\right) \d t \\ %\label{eq:upper_quad_2}\\
		&= A+\frac{M}{2}\|w\|^2 + B\|w\|,
	\end{align*}
where (i) follows from {\bf (A.1)} and from Cauchy--Schwarz, while (ii) follows from \eqref{eq:quadratic_growth}. This proves the upper bound on $f(w,z)$. Now take $v = cw$ for some $c \in (0,1]$ to be chosen later. With this choice, we proceed from Eq.~\eqref{eq:unit_speed} as follows:
\begin{align*}
	f(w,z) &= f(cw,z) + \int^1_c \ave{w, \nabla f(tw,z)} \d t \\
	&\stackrel{\rm (i)}{\ge}\int^1_c \frac{1}{t} \ave{tw, \nabla f(tw,z)} \d t \\ %\label{eq:lower_quad_1}\\
	&\stackrel{\rm (ii)}{\ge} \int^1_c \frac{1}{t}\left(mt^2\|w\|^2 - b\right) \d t \\ %\label{eq:lower_quad_2}\\
	&= \frac{m(1-c^2)}{2} \|w\|^2 + b \log c,
\end{align*}
where (i) uses the fact that $f \ge 0$, while (ii) uses the dissipativity property \eqref{eq:dissipation}. Taking $c = \frac{1}{\sqrt{3}}$, we get the lower bound in \eqref{eq:quadratic_sandwich}.
\end{proof}

\begin{proof}[Proof of Lemma~\ref{lm:mean_square_bounds}]  From \eqref{eq:Langevin}, it follows that
	\begin{align}
\E_\bz\|W_{k+1}\|^2
		&= \E_\bz\|W_k - \eta g(W_k,U_{\bz,k})\|^2 + \sqrt{\frac{8\eta}{\beta}}\E_\bz\ave{W_k - \eta g(W_k,U_{\bz,k}), \xi_k} + \frac{2\eta}{\beta}\E_\bz \|\xi_k\|^2 \nonumber \\
		&= \E_\bz\|W_k - \eta g(W_k,U_{\bz,k})\|^2 + \frac{2\eta d}{\beta},\label{eq:Langevin_L2_1}
	\end{align}
	where the second step uses independence of $W_k - g(W_k,U_{\bz,k})$ and $\xi_k$ and the unbiasedness property \eqref{eq:unbiased_grad} of the gradient oracle. We can further expand the first term in \eqref{eq:Langevin_L2_1}:
\begin{align}
&	\E_\bz\|W_k - \eta g(W_k,U_{\bz,k})\|^2 \nonumber\\
&= \E_\bz\left\|W_k - \eta \nabla F_\bz(W_k) \right\|^2 + 2 \eta \E_\bz\ave{W_k-\eta \nabla F_\bz(W_k),\nabla F_\bz(W_k)-g(W_k,U_{\bz,k})} \nonumber\\
& \qquad \qquad + \eta^2\E_\bz\|\nabla F_\bz(W_k)-g(W_k,U_{\bz,k})\|^2  \nonumber\\
&= \E_\bz\|W_k - \eta \nabla F_\bz(W_k) \|^2  + \eta^2\E_\bz\|\nabla F_\bz(W_k)-g(W_k,U_{\bz,k})\|^2, \label{eq:Langevin_L2_2}
\end{align}
where we have used \eqref{eq:unbiased_grad} once again. By \eqref{eq:SG_quadratic_growth}, the second term in \eqref{eq:Langevin_L2_2} can be upper-bounded by
\begin{align*}
	\E_\bz\|\nabla F_\bz(W_k)-g(W_k,U_{\bz,k})\|^2 \le 2\delta (M^2 \E_\bz\|W_k\|^2 + B^2),
\end{align*}
whereas the first term can be estimated as
\begin{align*}
	\E_\bz\|W_k - \eta \nabla F_\bz(W_k) \|^2 &= \E_\bz\|W_k\|^2 - 2 \eta \E_\bz\ave{W_k,\nabla F_\bz(W_k)} + \eta^2\E_\bz\| \nabla F_\bz(W_k)\|^2 \\
	&\le \E_\bz\|W_k\|^2 + 2\eta(b-m\E_\bz\|W_k\|^2) + 2\eta^2(M^2\E_\bz\|W_k\|^2+B^2) \\
	&= \left(1-2\eta m + 2\eta^2M^2\right)\E_\bz\|W_k\|^2 + 2\eta b + 2\eta^2B^2,
\end{align*}
where the inequality follows from the dissipativity condition \eqref{eq:dissipation} and the bound \eqref{eq:quadratic_growth} in Lemma~\ref{lm:quadratic_bounds}. Combining all of the above, we arrive at the recursion
\begin{align}\label{eq:Wk_recursion}
	\E_\bz\|W_{k+1}\|^2 \le (1-2\eta m+4\eta^2M^2) \E_\bz\|W_k\|^2 + 2\eta b + 4\eta^2B^2 + \frac{2\eta d}{\beta}.
\end{align}
Fix some $\eta \in (0,1 \wedge \frac{m}{2M^2})$. There are two cases to consider:
\begin{itemize}
	\item If $1-2\eta m + 4\eta^2M^2 \le 0$, then from \eqref{eq:Wk_recursion} it follows that
	\begin{align}
		\E_\bz\|W_{k+1}\|^2 &\le 2\eta b + 4\eta^2 B^2 + \frac{2\eta d}{\beta}\nonumber \\
		&\le \E_\bz\|W_0\|^2 + 2\left(b+2B^2 + \frac{d}{\beta}\right). \label{eq:Langevin_L2_3}
	\end{align}
	\item If $0 < 1 - 2\eta m + 4\eta^2M^2 < 1$, then iterating \eqref{eq:Wk_recursion} gives
	\begin{align}
		\E_\bz\|W_k\|^2 &\le (1-2\eta m + 4\eta^2M^2)^k \E_\bz\|W_0\|^2 + \frac{\eta b + 2\eta^2B^2 + \frac{\eta d}{\beta}}{\eta m - 2\eta^2M^2} \nonumber\\
		&\le \E_\bz\|W_0\|^2 + \frac{2}{m}\left(b + 2B^2 + \frac{d}{\beta}\right). \label{eq:Langevin_L2_4}
	\end{align}
\end{itemize}
The bound \eqref{eq:Langevin_bounded} follows from Eqs.~\eqref{eq:Langevin_L2_3} and \eqref{eq:Langevin_L2_4} and from the estimate
\begin{align}\label{eq:W0_bound}
	\E_\bz\|W_0\|^2 = \E\|W_0\|^2 \le \log \E e^{\|W_0\|^2} = \kappa_0,
\end{align}
which easily follows from the independence of $\bZ$ and $W_0$ and from Jensen's inequality.

We now analyze the diffusion \eqref{eq:diffusion}.  Let $Y(t) \deq \|W(t)\|^2$. Then It\^o's lemma gives
\begin{align*}
	\d Y(t) = -2\ave{W(t),\nabla F_\bz(W(t))} \d t + \frac{2d}{\beta} \d t+ \sqrt{\frac{8}{\beta}} W(t)^* \d B(t),
\end{align*}
where $W(t)^* \d B(t) \deq \sum^d_{i=1} W_i(t) \d B_i(t)$. This can be rewritten as
\begin{align}
	& 2me^{2mt} Y(t) \d t + e^{2mt} \d Y(t) \nonumber\\
	&= -2 e^{2mt} \ave{W(t),\nabla F_\bz(W(t))} \d t + 2me^{2mt} Y(t) \d t + \frac{2d}{\beta} e^{2mt} \d t + \sqrt{\frac{8}{\beta}} e^{2mt} W(t)^* \d B(t).\label{eq:diffusion_norm_squared_1}
\end{align}
Recognizing the left-hand side of \eqref{eq:diffusion_norm_squared_1} as the total It\^o derivative of $e^{2mt} Y(t)$, we arrive at
\begin{align}
	\d\big(e^{2mt}Y(t)\big) &= -2 e^{2mt} \ave{W(t),\nabla F_\bz(W(t))} \d t + 2me^{2mt} Y(t) \d t\nonumber\\
	& \qquad \qquad  + \frac{2d}{\beta} e^{2mt} \d t + \sqrt{\frac{8}{\beta}} e^{2mt} W(t)^* \d B(t),
\end{align}
which, upon integrating and rearranging, becomes
\begin{align}
	Y(t)
	&  = e^{-2mt}Y(0) -2\int^t_0 e^{2m(s-t)} \ave{W(s),\nabla F_\bz(W(s))} \d s \nonumber\\
	& \qquad \qquad + 2m \int^t_0 e^{2m(s-t)} Y(s) \d s + \frac{d}{m\beta}\left(1-e^{-2mt}\right)+ \sqrt{\frac{8}{\beta}} \int^t_0 e^{2m(s-t)} W(s)^* \d B(s). \label{eq:diffusion_norm_squared_2}
\end{align}
Now, using the dissipativity condition \eqref{eq:dissipation}, we can write
\begin{align*}
	-2\int^t_0 e^{2m(s-t)} \ave{W(s),\nabla F_\bz(W(s))} \d s  &\le 2\int^t_0 e^{2m(s-t)} \left(b-m Y(s)\right) \d s \\
	&= 2b\int^t_0 e^{2m(s-t)} \d s - 2m\int^t_0 e^{2m(s-t)} Y(s) \d s \\
	&= \frac{b}{m} \left(1-e^{-2mt}\right) - 2m\int^t_0 e^{2m(s-t)} Y(s) \d s.
\end{align*}
Substituting this into \eqref{eq:diffusion_norm_squared_2}, we end up with
\begin{align*}
	\|W(t)\|^2 &\le e^{-2mt}\|W(0)\|^2 + \frac{b + d/\beta}{m}\left(1-e^{-2mt}\right) +  \sqrt{\frac{8}{\beta}} \int^t_0 e^{2m(s-t)} W(s)^* \d B(s).
\end{align*}
Taking expectations and using the martingale property of the It\^o integral together with \eqref{eq:W0_bound}, we get \eqref{eq:2nd_moment_t}. Eq.~\eqref{eq:diffusion_bounded} follows from maximizing the right-hand side of \eqref{eq:2nd_moment_t} over all $t \ge 0$.
\end{proof}

\begin{proof}[Proof of Lemma~\ref{lm:diffusion_MGF}] For $L(t) = e^{\| W(t)\|^2}$, It\^o's lemma gives
	\begin{align*}
		\d L(t) = -2  \ave{W(t),\nabla F_\bz(W(t))} L(t) \d t + \frac{4}{\beta} L(t) \| W(t)\|^2 \d t + \frac{2d}{\beta} L(t) \d t + \sqrt{\frac{8}{\beta}}  L(t) W(t)^* \d B(t).
	\end{align*}
	Integrating, we obtain
	\begin{align*}
		L(t) &= L(0) +\int^t_0 \left(\frac{4}{\beta} \|W(s)\|^2 - 2 \ave{W(s),\nabla F_\bz(W(s))}\right) L(s) \d s \nonumber\\
		& \qquad \qquad + \frac{2d}{\beta}\int^t_0 L(s)\d s + \sqrt{\frac{8}{\beta}} \int^t_0 L(s) W(s)^* \d B(s).
	\end{align*}
From the dissipativity condition \eqref{eq:dissipation} and from the assumption that $\beta \ge 2/m$, it follows that
\begin{align*}
	\frac{4}{\beta} \|W(s)\|^2 - 2 \ave{W(s),\nabla F_\bz(W(s))} \le 2 b + \left(\frac{4}{\beta} - 2 m\right) \|W(s)\|^2 \le 2 b,
\end{align*}
hence
\begin{align*}
	L(t) &\le L(0) + 2\left(b+\frac{d}{\beta}\right)\int^t_0 L(s) \d s + \sqrt{\frac{8}{\beta}} \int^t_0 L(s) W(s)^* \d B(s).
\end{align*}
It can be shown (see, e.g., the proof of Corollary~4.1 in \citet{djellout_etal_TC}) that $\int^T_0 \E\|L(t)W(t)\|^2 \d t < \infty$ for all $T \ge 0$. Therefore, the It\^o integral $\int L(s)W(s)^* \d B(s)$ is a zero-mean martingale, so, taking expecations, we get
\begin{align*}
	\E[L(t)] &\le \E[L(0)] + 2\left(b+\frac{d}{\beta}\right)\int^t_0 \E[L(s] \d s \\
	&= e^{\kappa_0} + 2 \left(b+\frac{d}{\beta}\right)\int^t_0 \E[L(s)] \d s.
\end{align*}
Eq.~\eqref{eq:diffusion_MGF} then follows by an application of the Gronwall lemma. \end{proof}

\begin{proof}[Proof of Lemma~\ref{lm:relent_bound}] Let $p_\bz$ denote the density of $\pi_\bz$ with respect to the Lebesgue measure on $\Reals^d$:
	$$
	p_\bz(w) = \frac{e^{-\beta F_\bz(w)}}{\Lambda_\bz}, \qquad \text{where }\Lambda_\bz = \int_{\Reals^d} e^{-\beta F_\bz (w)} \d w.
	$$
Since $p_\bz > 0$ everywhere, we can write
	\begin{align}
		D(\mu_0 \| \pi_\bz) &= \int_{\Reals^d} p_0(w) \log \frac{p_0(w)}{p_\bz(w)} \d w \nonumber \\
		&= \int_{\Reals^d} p_0(w)\log p_0(w) \d w + \log \Lambda_\bz + \beta \int_{\Reals^d} p_0(w) F_\bz(w) \d w \nonumber \\
		&\le \log \|p_0\|_\infty + \log \Lambda_\bz + \beta \int_{\Reals^d} p_0(w) F_\bz(w) \d w.\label{eq:relent_bound_1}
	\end{align}
	We first upper-bound the partition function:
	\begin{align*}
		\Lambda_\bz &= \int_{\Reals^d} e^{-\beta F_\bz(w)} \d w \\
		&= \int_{\Reals^d} \exp\left(-\frac{\beta}{n}\sum^n_{i=1}f(w,z_i)\right) \d w \\
		&\le e^{\frac{1}{2}\beta b\log 3} \int_{\Reals^d} e^{-\frac{m\beta\|w\|^2}{3}}\d w \\
		&= 3^{\beta b/2} \left(\frac{3\pi}{m\beta}\right)^{d/2},
	\end{align*}
	where the inequality follows from Lemma~\ref{lm:quadratic_bounds}. Thus,
	\begin{align}\label{eq:partition_function_bound}
		\log \Lambda_\bz \le \frac{d}{2}\log \frac{3\pi}{m\beta} + \frac{\beta b}{2}\log 3.
	\end{align}
Moreover, invoking Lemma~\ref{lm:quadratic_bounds} once again, we have
	\begin{align}\label{eq:F_bound}
	 F_\bz(w) &= \frac{1}{n}\sum^n_{i=1} f(w,z_i) \le  \frac{M}{3}\|w\|^2 + B\|w\|+A.
	 \end{align}
	 Therefore,
	 \begin{align}
		 \int_{\Reals^d} F_\bz(w) p_0(w) \d w &\le \int_{\Reals^d} \mu_0(\d w)\left(\frac{M}{3} \|w\|^2 + B\|w\| + A\right) \nonumber \\
		 &\le \frac{M}{3}\kappa_0+ B\sqrt{\kappa_0} + A. \label{eq:expected_F_bound}
	\end{align}
	Substituting \eqref{eq:partition_function_bound}, \eqref{eq:F_bound}, and \eqref{eq:expected_F_bound} into \eqref{eq:relent_bound_1}, we get \eqref{eq:relent_bound}.
\end{proof}

\begin{proof}[Proof of Lemma~\ref{lm:Wass_continuity}] The proof is a minor tweak of the proof of Proposition~1 in \cite{Polyanskiy_Wu_Wasserstein}; we reproduce it here to keep the presentation self-contained. Without loss of generality, we assume that $\sigma^2 < \infty$, otherwise the bound holds trivially. For any two $v,w \in \Reals^d$, we have
	\begin{align}
		g(w)-g(v) &= \int^1_0 \ave{w-v, \nabla g((1-t)v + t w)} \d t \nonumber\\
		&\le \int^1_0 \| \nabla g((1-t)v+tw) \| \, \| w-v \| \, \d t \nonumber\\
		&\le \int^1_0 \left( c_1(1-t)\|v\| + c_1 t \| w\| + c_2\right) \, \| w-v\| \, \d t \nonumber\\
		&= \left(\frac{c_1}{2}\|v\|+\frac{c_1}{2}\|w\|+c_2\right) \, \|w-v\|, \label{eq:Wass_cont_1}
	\end{align}
	where we have used Cauchy--Schwarz and the growth condition \eqref{eq:PW_regular}. Now let $\bP$ be the coupling of $\mu$ and $\nu$ that achieves $\Wass_2(\mu,\nu)$. That is, $\bP = \cL((W,V))$ with $\mu =\cL(W)$, $\nu = \cL(V)$, and
\begin{align*}
	\Wass^2_2(\mu,\nu) = \E_\bP \| W - V \|^2.
\end{align*}
Taking expectations in \eqref{eq:Wass_cont_1}, we have
\begin{align*}
	\int_{\Reals^d} g \d\mu - \int_{\Reals^d} g\d\nu  &= \E_\bP[g(W)-g(V)] \\
	&\le \sqrt{\E_\bP\left(\frac{c_1}{2}\|W\|+\frac{c_1}{2}\|V\|+c_2\right)^2} \cdot \sqrt{\E_\bP [\|W-V\|^2]} \\
	&\le \left( \frac{c_1}{2}\sqrt{\E \|W\|^2} + \frac{c_1}{2}\sqrt{\E\|V\|^2}+c_2\right) \cdot \Wass_2(\mu,\nu) \\
	&= \left(c_1 \sigma + c_2\right) \Wass_2(\mu,\nu).
\end{align*}
Interchanging the roles of $\mu$ and $\nu$, we complete the proof.
\end{proof}

\section{Proof of Lemma~\ref{lm:Girsanov_bound}}
\label{app:diffusion_approx}

Conditioned on $\bZ=\bz$, $\{W_k\}^\infty_{k=0}$ is a time-homogeneous Markov process. Consider the following continuous-time interpolation of this process: 
\begin{align}\label{eq:EM_interpol}
	\overline{W}(t) = W_0 - \int^t_0 g(\overline{W}(\lfloor s/\eta \rfloor \eta),\overline{U}_\bz(s)) \d s + \sqrt{\frac{2}{\beta}}\int^t_0 \d B(s), \qquad t \ge 0
\end{align}
where $\overline{U}_\bz(t) \equiv U_{\bz,k}$ for $t \in [k\eta,(k+1)\eta)$. Note that, for each $k$, $\overline{W}(k\eta)$ and $W_k$ have the same probability law $\mu_{\bz,k}$.  Moreover, by a result of \citet{gyongy_SDE}, the process $\overline{W}(t)$ has the same one-time marginals as the It\^o process
\begin{align*}
V(t) = W_0 - \int^t_0 g_{\bz,s}(V(s)) \d s + \sqrt{\frac{2}{\beta}}\int^t_0 \d B(s)
\end{align*}
with
\begin{align}\label{eq:Gyongy_drift}
	g_{\bz,t}(v) := \E_\bz\Big[g(\overline{W}(\lfloor t/\eta \rfloor \eta),\overline{U}_{\bz}(t))\Big|\overline{W}(t)=v\Big].
\end{align}
Crucially, $V(t)$ is a Markov process, while $\overline{W}(t)$ is not. Let $\bP^t_V \deq \cL\big(V(s) : 0 \le s \le t\big|\bZ=\bz\big)$ and $\bP^t_W \deq \cL\big(W(s) : 0 \le s \le t\big|\bZ=\bz\big)$. The Radon--Nikodym derivative of $\bP^t_W$ w.r.t.\ $\bP^t_V$ is given by the Girsanov formula
\begin{align}\label{eq:Girsanov}
	\frac{\d \bP^t_W}{\d \bP^t_V}(V) = \exp\left\{ \frac{\beta}{2} \int^t_0 \left(\nabla F_\bz(V(s))-g_{\bz,s}(V(s))\right)^* \d B(s) - \frac{\beta}{4}\int^t_0 \|  \nabla F_\bz(V(s))-g_{\bz,s}(V(s)) \|^2 \d s\right\}
\end{align}
(see, e.g., Sec.~7.6.4 in \citet{Liptser_Shiryaev_vol1}). Using \eqref{eq:Girsanov} and the martingale property of the It\^o integral, we have
\begin{align*}
	D(\bP^t_V\|\bP^t_W) &= -\int \d \bP^t_V \log \frac{\d \bP^t_W}{\d \bP^t_V} \\
	&= \frac{\beta}{4}\int^t_0 \E_\bz \left\|  \nabla F_\bz(V(s))-g_{\bz,s}(V(s)) \right\|^2 \d s \\
	&= \frac{\beta}{4}\int^t_0 \E_\bz\left\|  \nabla F_\bz(\overline{W}(s))-g_{\bz,s}(\overline{W}(s)) \right\|^2\, \d s,
\end{align*}
where the last line follows from the fact that $\cL(\overline{W}(s)) = \cL(V(s))$ for each $s$. 

Now let $t = k\eta$ for some $k \in \Naturals$. Then, using the definition \eqref{eq:Gyongy_drift} of $g_{\bz,s}$, Jensen's inequality, and the $M$-smoothness of $F_\bz$, we can write
\begin{align}
	D(\bP^{k\eta}_V \| \bP^{k\eta}_W) &= \frac{\beta}{4}\sum^{k-1}_{j=0} \int^{(j+1)\eta}_{j\eta} \E_\bz\left\| \nabla F_\bz(\overline{W}(s))-g_{\bz,s}(\overline{W}(s))\right\|^2  \d s \nonumber\\
	&\le \frac{\beta}{2} \sum^{k-1}_{j=0} \int^{(j+1)\eta}_{j\eta} \E_\bz \left\| \nabla F_\bz(\overline{W}(s))-\nabla F_\bz(\overline{W}(\lfloor s/\eta \rfloor \eta)) \right\|^2 \d s \nonumber\\
  & \qquad \qquad + \frac{\beta}{2}\sum^{k-1}_{j=0}\int^{(j+1)\eta}_{j\eta} \E_\bz \left\| \nabla F_\bz(\overline{W}(\lfloor s/\eta \rfloor \eta))-g(\overline{W}(\lfloor s/\eta \rfloor \eta),\overline{U}_\bz(s))\right\|^2 \d s \nonumber \\
	& \le \frac{\beta M^2}{2} \sum^{k-1}_{j=0} \int^{(j+1)\eta}_{j\eta} \E_\bz \left\| \overline{W}(s)-\overline{W}(\lfloor s/\eta \rfloor \eta) \right\|^2 \d s \nonumber\\
  & \qquad \qquad + \frac{\beta}{2}\sum^{k-1}_{j=0}\int^{(j+1)\eta}_{j\eta} \E_\bz \left\| \nabla F_\bz(\overline{W}(\lfloor s/\eta \rfloor \eta))-g(\overline{W}(\lfloor s/\eta \rfloor \eta),\overline{U}_\bz(s))\right\|^2 \d s.
	\label{eq:Girsanov_1}
\end{align}
We first estimate the first summation in \eqref{eq:Girsanov_1}. Consider some $s \in [j\eta,(j+1)\eta)$. From \eqref{eq:EM_interpol}, we have
\begin{align*}
	\overline{W}(s) - \overline{W}(j\eta) &= -(s-j\eta)g(W_j,U_{\bz,j})  + \sqrt{\frac{2}{\beta}}\left(B(s)-B(j\eta)\right) \\
	&= -(s-j\eta) \nabla F_\bz(W_j) + (s-j\eta)\left(\nabla F_\bz(W_j)-g(W_j,U_{\bz,j})\right)  + \sqrt{\frac{2}{\beta}}\left(B(s)-B(j\eta)\right).
\end{align*}
Therefore, using Lemmas~\ref{lm:quadratic_bounds} and \ref{lm:mean_square_bounds} and the gradient noise assumption {\bf (A.4)}, we arrive at
\begin{align*}
&	\E_\bz\| \overline{W}(s) - \overline{W}(j\eta) \|^2 \nonumber\\
&\qquad \le 3\eta^2 \E_\bz \|\nabla F_\bz(W_j) \|^2 + 3\eta^2 \E_\bz \| \nabla F_\bz(W_j)-g(W_j,U_{\bz,j}) \|^2 + \frac{6\eta d}{\beta} \\
&\qquad \le 12\eta^2  \left(M^2 \E_\bz\|W_j\|^2 + B^2\right) + \frac{6\eta d}{\beta} \\
&\qquad \le 12\eta^2  \left(M^2 \left( \kappa_0 + 2\left(1 \vee \frac{1}{m}\right)  \left(b+2B^2+\frac{d}{\beta}\right) \right ) + B^2\right) + \frac{6\eta d}{\beta}.
\end{align*}
Consequently,
\begin{align}
	& \sum^{k-1}_{j=0} \int^{(j+1)\eta}_{j\eta} \E_\bz \left\| \overline{W}(s)-\overline{W}(\lfloor s/\eta \rfloor \eta) \right\|^2 \d s \nonumber\\
	& \qquad \le 12 \left(M^2 \left( \kappa_0 + 2\left(1 \vee \frac{1}{m}\right)  \left(b+2B^2+\frac{d}{\beta}\right) \right ) + B^2\right)k\eta^3 + \frac{6 d}{\beta} k \eta^2 \nonumber \\
	& \qquad \le  \left( 12 \left(M^2 \left(\kappa_0 + 2\left(1 \vee \frac{1}{m}\right)  \left(b+2B^2+\frac{d}{\beta}\right) \right ) + B^2\right) + \frac{6 d}{\beta}\right) \cdot k\eta^2 \nonumber \\
	& \qquad =: 6\left(2C_0 + \frac{d}{\beta}\right) \cdot k\eta^2.\label{eq:Girsanov_term1}
\end{align}
Similarly, the second summation on the right-hand side of \eqref{eq:Girsanov_1} can be estimated as follows:
\begin{align}
  &\sum^{k-1}_{j=0}\int^{(j+1)\eta}_{j\eta} \E_\bz \left\| \nabla F_\bz(\overline{W}(\lfloor s/\eta \rfloor \eta))-g(\overline{W}(\lfloor s/\eta \rfloor \eta),\overline{U}(s))\right\|^2 \d s \nonumber\\
	&\qquad = \eta\sum^{k-1}_{j=0} \E_\bz \left\| \nabla F_\bz(W_j)-g(W_j,U_{\bz,j})\right\|^2  \nonumber \\
	&\qquad \le \eta\delta \sum^{k-1}_{j=0} 2\left(M^2 \E_\bz\|W_j\|^2 + B^2\right) \nonumber \\
	&\qquad \le 2M^2 \left( \kappa_0 + 2\left(1 \vee \frac{1}{m}\right)  \left(b+2B^2+\frac{d}{\beta}\right) \right) k\eta\delta + 2\delta B^2 k\eta \nonumber\\
	&\qquad = 2\left( M^2 \left( \kappa_0 + 2\left(1 \vee \frac{1}{m}\right)  \left(b+2B^2+\frac{d}{\beta}\right) \right)  + B^2 \right) \cdot k\eta \delta \nonumber\\
	&\qquad = 2C_0 \cdot k\eta\delta. \label{eq:Girsanov_term2}
\end{align}
Substituting Eqs.~\eqref{eq:Girsanov_term1} and \eqref{eq:Girsanov_term2} into \eqref{eq:Girsanov_1}, we obtain
\begin{align*}
	D(\bP^{k\eta}_V \| \bP^{k\eta}_W) \le 6\left(\beta M^2C_0 + M^2d\right) \cdot k\eta^2 + \beta C_0 \cdot k\eta\delta.
 \end{align*}
Now, since $\mu_{\bz,k} = \cL(\overline{W}(k\eta)|\bZ=\bz)$ and $\nu_{\bz,k\eta} = \cL(W(k\eta)|\bZ=\bz)$, the data-processing inequality for the KL divergence gives
\begin{align*}
	D(\mu_{\bz,k}\|\nu_{\bz,k\eta}) &\le D(\bP^{k\eta}_V \| \bP^{k\eta}_W) \\
	&\le 6\left(\beta M^2C_0 + M^2d\right) \cdot k\eta^2 + \beta C_0 \cdot k\eta\delta \\
	&=: C_1 k\eta^2 + \beta C_0 k\eta\delta.
\end{align*}

\section{Proof of Proposition~\ref{prop:LSI}}
\label{app:LS_proof}

To establish the log-Sobolev inequality, we will use the Lyapunov function criterion of  \citet{Cattiaux_etal_Lyapunov}, reproduced as Proposition~\ref{prop:Lyapunov_LSI} in Appendix~\ref{app:background}. 

We will apply this proposition to the Gibbs distribution $\pi_\bz$ for some $\bz \in \sZ^n$, so that $H = \beta F_\bz$ and
$$
\cL g = \Delta g - \beta \ave{\nabla F_\bz, \nabla g}.
$$
We consider the same Lyapunov function $V(w) = e^{m\beta \|w\|^2/4}$ as in Appendix~\ref{app:spectral_gap}. From Eq.~\eqref{eq:LV}, $V$ evidently satisfies \eqref{eq:Lyapunov_LSI} with $\kappa$ and $\gamma$ given in \eqref{eq:kappa_gamma}, i.e., the first condition of Proposition~\ref{prop:Lyapunov_LSI} is satisfied. Moreover, $\pi_\bz$ satisfies a Poincar\'e inequality with constant $c_{\rm P} \le 1/\lambda_*$.  Thus, the second condition is also satisfied. 
Finally, by the $M$-smoothness assumption {\bf (A.2)}, $\nabla^2 F_\bz \succeq -M I_d$, so the third condition of Proposition~\ref{prop:Lyapunov_LSI} is satisfied with $K = \beta M$. Consequently, the constants $C_1$ and $C_2$ in \eqref{eq:C1_C2} are given by
\begin{align}\label{eq:C1_C2_explicit}
	C_1 = \frac{2m^2 + 8M^2}{m^2 M \beta} \qquad \text{and} \qquad C_2 \le \frac{6M(d+\beta)}{m},
\end{align}
where we have also used the estimate \eqref{eq:equilibrium_2nd_moment} to upper-bound $C_2$. Therefore, from Proposition~\ref{prop:Lyapunov_LSI} and from \eqref{eq:C1_C2_explicit} it follows that $\pi_\bz$ satisfies a logarithmic Sobolev inequality with
\begin{align*}
	c_{\rm LS} &\le \frac{2m^2 + 8M^2}{m^2 M \beta} + \frac{1}{\lambda_*} \left( \frac{6M(d+\beta)}{m} + 2\right).
\end{align*}

\end{document}